\documentclass{article}
\usepackage{microtype} % recommended for figures and better typesetting
\usepackage{graphicx} % recommended for figures and better typesetting
\usepackage{booktabs} % for professional tables
\usepackage{iclr26/iclr2026_conference}
\usepackage{times}
\usepackage{algorithm}
\usepackage[noend]{algorithmic}
\usepackage[titlenumbered,ruled,noend,algo2e]{algorithm2e}
\usepackage{amsfonts}
\usepackage{amsmath}
\usepackage{amssymb}
\usepackage{amsthm} % theorems
\usepackage{thmtools} % restatable
\usepackage{array} % to align right
\usepackage{colortbl}
\usepackage[inline]{enumitem} % inline enumeration
\usepackage{eso-pic} % used by \AddToShipoutPicture
\usepackage{forloop}
\usepackage{hyperref, cleveref}

\usepackage[bbgreekl]{mathbbol}
\usepackage{mathtools}

\usepackage{multicol}
\usepackage{pgfplots}
\usepackage{pgfplotstable}
\pgfplotsset{compat=1.18}
\usepgfplotslibrary{groupplots}
\usepackage{subcaption}
\usepackage{tabularx}
\usepackage[most]{tcolorbox}
\usepackage{thm-restate}
\usepackage{tikz}
\usetikzlibrary{fit,positioning,calc,shapes,backgrounds,external,patterns}
\usepackage{url}
\usepackage{xspace}
\usepackage{xcolor} % color for editing
\PassOptionsToPackage{dvipsnames}{xcolor}
\usepackage{xfrac} % for \sfrac
\usepackage{wrapfig}

\setlist[itemize]{leftmargin=*}

\newcommand{\CG}{\mathcal{G}\xspace}

\newcommand{\CV}{\mathcal{V}\xspace}
\newcommand{\CE}{\mathcal{E}\xspace}

\newcommand{\name}{\textsc{Eugene}\xspace}
\newcommand{\madam}{\textsc{M-Adam}\xspace}

\newcommand{\greed}{\textsc{Greed}\xspace}
\newcommand{\gedgnn}{\textsc{GedGnn}\xspace}

\newcommand{\gennastar}{{Genn-A*}\xspace}
\newcommand{\eric}{\textsc{Eric}\xspace}
\newcommand{\egsc}{\textsc{Egsc}\xspace}
\newcommand{\graphedx}{\textsc{GraphEdx}\xspace}
\newcommand{\gmn}{\textsc{Gmn-Embed}\xspace}
\newcommand{\hmn}{\textsc{H2mn}\xspace}
\newcommand{\graphsim}{\textsc{GraphSim}\xspace}
\newcommand{\gotsim}{\textsc{GOTSim}\xspace}
\newcommand{\simgnn}{\textsc{SimGNN}\xspace}
\newcommand{\btight}{\textsc{Branch-Tight}\xspace}
\newcommand{\adjip}{\textsc{Adj-Ip}\xspace}
\newcommand{\lpged}{\textsc{F1}\xspace}
\newcommand{\compact}{\textsc{Compact-MIP}\xspace}
\newcommand{\ipfp}{\textsc{IPFP}\xspace}
\newcommand{\gmsm}{\textsc{GMSM}\xspace}
\newcommand{\fugal}{\textsc{Fugal}\xspace}

\newcommand{\aids}{AIDS\xspace}
\newcommand{\molhiv}{Molhiv\xspace}
\newcommand{\code}{Code2\xspace}
\newcommand{\mutag}{Mutag\xspace}
\newcommand{\imdb}{IMDB\xspace}
\newcommand{\coil}{COIL-DEL\xspace}
\newcommand{\triangles}{Triangles\xspace}

\newcommand{\nodeins}{\ensuremath{v_\mathsf{ins}}\xspace}
\newcommand{\nodedel}{\ensuremath{v_\mathsf{del}}\xspace}

\newcommand{\edgeins}{\ensuremath{e_\mathsf{ins}}\xspace}
\newcommand{\edgedel}{\ensuremath{e_\mathsf{del}}\xspace}

\newcommand*{\NPhard}{$\mathbf{NP}$-hard\xspace}
\newcommand*{\NPhardness}{$\mathbf{NP}$-hardness\xspace}

\newcommand*{\APXhard}{$\mathbf{APX}$-hard\xspace}

\newcommand{\pkl}[1]{#1}
\newcommand{\pke}[1]{#1}
\newcommand{\pkn}[1]{#1}
\newcommand{\pk}[1]{#1}
\newcommand{\pki}[1]{#1}
\newcommand{\iclr}[1]{#1}
\newcommand{\nips}[1]{#1}
\newtheorem{lemma}{Lemma}
\newtheorem{definition}{Definition}

\renewcommand{\cite}{\citep} % prit parentheses by default
\title{\name: Explainable Structure-aware Graph Edit Distance Estimation with Generalized Edit Costs} 

\author{Aditya Bommakanti \\
Department of Computer Science\\
IIT Delhi \\
\texttt{adityabommakanti2002@gmail.com} \\
\And
Harshith Reddy Vonteri \\
Department of Computer Science \\
IIT Delhi \\\
\texttt{harshithreddyvonteri@gmail.com} \\
\AND
Sayan Ranu \\
Department of Computer Science\\
IIT Delhi \\
\texttt{sayanranu@cse.iitd.ac.in}
\AND
Panagiotis Karras \\
Department of Computer Science\\
University of Copenhagen \& Aarhus University \\
\texttt{piekarras@gmail.com}
}

\iclrfinalcopy
\begin{document}
% \twocolumn[
% \icmltitle{\name: Explainable Unsupervised Approximation of Graph Edit Distance with Generalized Edit Costs}
\maketitle

% \begin{icmlauthorlist}
% \icmlauthor{Firstname1 Lastname1}{equal,yyy}
% \icmlauthor{Firstname2 Lastname2}{equal,yyy,comp}
% \icmlauthor{Firstname3 Lastname3}{comp}
% \end{icmlauthorlist}



% % You may provide any keywords that you
% % find helpful for describing your paper; these are used to populate
% % the "keywords" metadata in the PDF but will not be shown in the document

% \printAffiliationsAndNotice{}
% \title[\name]{}
% \input{00.authors}
\begin{abstract}
The need to identify graphs with small structural distances from a query arises in domains such as biology, chemistry, recommender systems, and social network analysis. Among several methods for measuring inter-graph distance, Graph Edit Distance (GED) is preferred for its comprehensibility, though its computation is hindered by \NPhardness. Optimization based heuristic methods often face challenges in providing accurate approximations. State-of-the-art GED approximations predominantly utilize neural methods, which, however:
\begin{enumerate*}[label=(\roman*)]
    \item lack an \emph{explanatory} edit path corresponding to the approximated GED;\label{item:nopath}
    \item require the NP-hard generation of ground-truth GEDs for training; and\label{item:generation}
    \item necessitate separate training on each dataset.\label{item:training}
\end{enumerate*}
In this paper, we propose \name, an efficient, algebraic, \pke{and structure-aware} optimization based method that estimates GED and also provides edit paths corresponding to the estimated cost. Extensive experimental evaluation demonstrates that \name achieves state-of-the-art GED estimation with superior scalability across diverse datasets and generalized cost settings.
\end{abstract}
%\input{00.ccs_keywords}
%\maketitle
% \maketitle
%\input{sections/00.abstract}
\vspace{-0.1in}
\section{Introduction and Related work}\label{sec:intro}
\vspace{-0.1in}
\pkl{\emph{Graph Edit Distance (GED)} quantifies the dissimilarity between a pair of graphs~\cite{graphsim,graphotsim,simgnn,ranjan&al22}. It finds application in identifying the graph in a collection most similar to a query graph. Given graphs~$\mathcal{G}_1$ and~$\mathcal{G}_2$, GED is the minimum cost to transform~$\mathcal{G}_1$ into~$\mathcal{G}_2$ through \emph{edit operations}, rendering~$\mathcal{G}_1$ isomorphic to~$\mathcal{G}_2$. These operations comprise the addition and deletion of edges and nodes and the replacement of their labels, each linked to a cost. \Cref{fig:ged_example} presents an example. GED computation is \NPhard~\cite{ar:1} and \APXhard~\cite{apxhard}, hence a challenging task.}
%\pki{Moreover, GED is APX-hard~\cite{Lin1994HardnessOA}, rendering it inapproximable within polynomial time.}
%\looseness=-1 % no need now.

\begin{figure}[H]
\centering
\vspace{-2mm}
\includegraphics[width=3.3in]{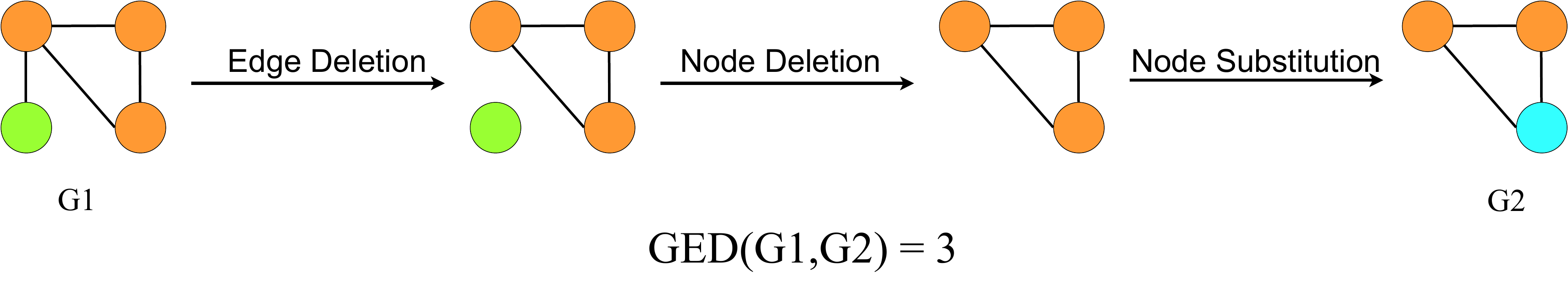}
\vspace{-2mm}
\caption{\pkl{An \emph{edit path} between graphs~$\mathcal{G}_1$ and~$\mathcal{G}_2$ with GED~3; each edit operation costs~1.}}\label{fig:ged_example}
\vspace{-4mm}
\end{figure}

\pke{Owing to the problem's hardness, several algorithms approximate GED~\cite{Blumenthal}. \textit{Optimization} based heuristic GED estimation methods employ strategies such as transformations to the linear-sum assignment problem with error correction or constraints (e.g., \textsc{NODE}~\cite{adjip}, \btight~\cite{btight}) and linear-programming relaxations of mixed integer programming (\textsc{MIP}) formulations (e.g., \lpged~\cite{LEROUGE2017254}, \adjip~\cite{adjip}, \compact~\cite{blumenthal2020exact}). Still, these approaches often afford only limited approximation accuracy.}
%, motivating the development of neural architectures aimed at improving GED prediction quality.
%and index structures~\cite{ctree,stars,ged1,ged2} 
%have been designed in the literature. 
%\noindent\textbf{Unsupervised Methods:} \pk{One approach~\cite{ar:3} formulates the exact GED computation as a binary linear programming problem. Another method, \textsc{Branch}~\cite{ar:4}, uses the linear-sum assignment problem with error-correction (\textsc{Lsape}) to process the search space and heuristically compute GED, achieving a good tradeoff between accuracy and time. Bipartite-matching methods map nodes and local structures among two graphs~\cite{ar:2,pr:1} by solving an \emph{assignment problem}. \pki{However, currently, unsupervised methods that target the exact GED have limited scalability, while those that seek sub-optimal solutions suffer from limited accuracy.}}

%\textbf{Supervised Methods:} 
\pke{Recent works have evinced that graph neural networks (GNNs) can achieve state-of-the-art accuracy in approximating GED~\cite{graphedx, ranjan&al22, WangCVPR21, simgnn, graphsim, graphotsim, icmlged, h2mn, gedgnn}. The general pipeline in this paradigm is to train a GNN-based architecture on a set of graph pairs along with their true GED distance. Some techniques also require the node mapping corresponding to the GED~\cite{gedgnn,WangCVPR21}.}

%In this approach, a GNN-based architecture is developed to learn a prediction model from a training set of graph pairs and their distances, enabling the prediction of GED for unseen pairs. 
%Among the prominent algorithms in this domain, \textsc{Greed}~\cite{ranjan&al22} utilizes \pki{siamese} GNNs with a specific inductive bias to learn GED while maintaining its metric properties. \textsc{Genn-A*}~\cite{WangCVPR21} combines GNNs with the A* algorithm to predict distances and generate edit paths for graphs up to 10 nodes in size. \textsc{H$^2$MN}~\cite{h2mn} leverages a hierarchical hypergraph matching network to learn graph similarity. Other methods, including \textsc{SimGNN}~\cite{simgnn}, \textsc{GotSIM}~\cite{graphotsim}, \textsc{GraphSIM}~\cite{graphsim}, and \textsc{GMN}~\cite{icmlged}, also employ GNNs to estimate GED.

\pke{Although they afford superior accuracy, neural approaches suffer from notable drawbacks:}

\begin{itemize}
\item \textbf{Reliance on \NPhard ground truth}: \pke{Generating training data, i.e., \textit{true} GEDs of graph pairs, is prohibitively costly for large graphs, as GED computation is \NPhard. Training data are thus limited to graphs of at most~$25$ nodes, undermining generalizability to larger ones (\S~\ref{sec:exp}).}
\item \textbf{Lack of interpretability}: \pki{\pkn{Most of them} furnish a GED between two graphs \emph{but not} an edit path that entails it;} \pke{such edit paths reveal crucial functions of protein complexes~\cite{biology1}, image alignment~\cite{cv1}, and gene regulatory pathways~\cite{biology3}. Some neural methods, e.g., \textsc{GedGnn}~\cite{gedgnn} and \textsc{Genn}-$A^*$~\cite{WangCVPR21} offer interpretability, albeit at the expense of accuracy and/or scalability, as we show in \S~\ref{sec:exp}.}
\item \textbf{Lack of generalizability}: \pke{Neural approximators do not generalize across datasets. For datasets across different domains (such as chemical compounds vs. function-call graphs), the node label set changes. As the number of parameters in a GNN is a function of the feature dimension in each node, a GNN trained on one domain cannot transfer to another,
%Even when two datasets belong to the same domain, as we will show later in \S~\ref{app:generalizability}, due to distribution shifts in structure and node label distributions, prediction accuracy suffers. 
necessitating \emph{separate} training for each dataset. As training data generation is \NPhard, the pipeline is resource-intensive.}
%\vspace{-0.05in}
%\item {\bf Environmental impact:} Neural solvers impose intense computational resource requirements, leading to increased greenhouse gas emissions, due to their reliance on GPUs for both training and inference. %This dependency results in high energy consumption and  %As AI deployments continue to scale, these environmental implications become increasingly critical, highlighting the need for more efficient and sustainable approaches in machine learning.
\end{itemize}
%\looseness=-1 % no need any more.

\pk{In this paper, we present an optimization based algebraic method called \name: \underline{E}xplainable Str\underline{u}cture-aware \underline{G}raph \underline{E}dit Dista\underline{n}c\underline{e}, which achieves state-of-the-art accuracy and is:
\begin{enumerate*}[label=\textbf{(\arabic*)}, itemjoin={{; }}, itemjoin*={{; and }}]
\item \textit{optimization} based heuristic, hence does not require training
\item \textit{CPU-bound}, therefore unshackled from GPU requirements and resultant greenhouse emissions
\item \textit{interpretable}.
\end{enumerate*}
%Finally, the output of \name is the edit path, from which the GED is computed. Hence, it is fully \textit{interpretable}. 
The innovations empowering these properties are as follows:
% \vspace{-0.1in}
\begin{itemize}
%\item \textbf{Equivalence of GED with graph alignment:} We establish a fundamental connection between two pivotal graph theory problems, namely Unrestricted Graph Alignment (UGA) and Graph Edit Distance (GED), by proving that an UGA corresponds to a specific instance of GED.
\vspace{-0.05in}
\item \textbf{Optimization problem formulation:} We cast the GED computation problem as an optimization problem extending over Unrestricted Graph Alignment (UGA), \pke{grounded on adjacency matrices}, over the space of all possible node alignments, represented via \textit{permutation matrices}; \pki{this formulation facilitates an optimization based solution}, eschewing the need for ground-truth data generation and data-specific training.
\vspace{-0.05in}
%\item \textbf{Interpretability with convergence guarantee:} To approximate GED, \name minimizes a function over the set of \emph{doubly stochastic} matrices, leading to a convex optimization problem that can be solved by \textsc{Adam}~\cite{kingma14}. We further refine the approximation by exhorting the doubly stochastic matrix to obtain a \emph{quasi-permutation matrix} form. By operating directly on matrices, \name yields a GED approximation \emph{explainable} via a node-to-node correspondence. We also devise a customized parameter optimization strategy that guarantees convergence.
\item \textbf{Interpretability:} To approximate GED, \name minimizes a function over the set of \emph{doubly stochastic} matrices, leading to a convex optimization problem that can be solved by \textsc{Adam}~\cite{kingma14}. We further refine the approximation by exhorting the doubly stochastic matrix \iclr{using permutation inducing regularizers and inverse relabelling strategy}. By operating directly on matrices, \name yields a GED approximation \emph{explainable} via a node-to-node correspondence.% We also devise a customized parameter optimization strategy that guarantees convergence.
%\item \textbf{Convergence guarantee:} Optimizing over a quasi-permutation matrix renders the problem non-convex. Therefore, we devise a customized parameter optimization algorithm for \name that guarantees convergence.
%\looseness=-1 % no need.
% \item \textbf{Rounding Algorithms for Node Alignment:} We incorporate two rounding algorithms to generate a permutation matrix that represents the node alignment between graphs. The resulting node alignment yields the approximate GED, thus providing an interpretable solution to the problem.
\vspace{-0.05in}
\item \textbf{Experimental evaluation:} Extensive experiments encompassing 15 state-of-the-art baselines over 9 datasets and 3 combinations of edits costs establish that \name consistently achieves superior accuracy in GED approximation. Notably, \name, does not rely on training data and thus offers a resource-efficient, GPU-free execution pipeline, which exhibits up to 30 times lower carbon emissions than its neural counterparts.
\end{itemize}}
\vspace{-0.1in}
\vspace{-0.1in}
\section{Preliminaries and Problem Formulation}\label{sec:notation}
%In this section, we introduce the concepts central to our work and formulate the problem of approximating graph edit distance.
\vspace{-0.1in}
\begin{definition}[Graph]
A node-labeled undirected graph is a triple~$\mathcal{G}(\mathcal{V}, \mathcal{E}, \mathcal{L})$ where~$\mathcal{V} = [n] \equiv\{1,\ldots,n\}$ is the node set, $\mathcal{E} \subseteq [n]\times[n]$ is the edge set, and~$\mathcal{L}: \mathcal{V} \rightarrow \Sigma$ is a labeling function that maps nodes to labels, where~$\Sigma$ is the set of all labels.
\end{definition}
\vspace{-2mm}
The \textit{adjacency matrix} of~$\mathcal{G}$ is~$A = [a_{i,j}]_{i,j \in [n]} \in \{0,1\}^{n \times n}$ such that~$a_{ij} = a_{ji} = 1$ if and only if~$(i,j)\in E$. We use~$\mathbf{1}$ to denote an all-ones vector, $J$ to denote an all-ones square matrix, and~$O$ to denote an all-zero square matrix.
%We let dimensions be inferred from the employing equations.

\begin{definition}[Permutation and Doubly Stochastic Matrices]\label{def:matrix}
A permutation matrix of size~$n$ is a binary-valued matrix~$\mathbb{P}^{n} = \{P \in \{0, 1\}^{n \times n} : P\mathbf{1} = \mathbf{1}, P^{T}\mathbf{1} = \mathbf{1}\}$. A doubly stochastic matrix of size~$n$ is a real-valued matrix~$\mathbb{W}^{n} = \{W \in [0, 1]^{n \times n} : W\mathbf{1} = \mathbf{1}, W^{T}\mathbf{1} = \mathbf{1}\}$.
\end{definition}

We define a $\emph{quasi-permutation matrix}$ as a matrix that is \textit{almost} a permutation matrix.
\begin{definition}[Entry-wise norm] Let~$A \!=\! [a_{ij}]_{i,j \in [n]} \!\in\! \mathbb{R}^{n \times n}$ and $p \!\in\! \mathbb{N}^+ \cup \{\infty\}$. We define the entry-wise $p$-norm of $A$ as~$\|A\|_p = \left( \sum_{i=1}^n \sum_{j=1}^n |a_{ij}|^{p} \right)^{\sfrac1p}$ for $p \!\in\! \mathbb{N}+$, and~$\|A\|_\infty \!=\! \max_{i,j} |a_{i,j}|$. We denote the entry-wise $2$-norm (i.e., the  \emph{Frobenius} norm) as $\|\cdot\|_F$.
\end{definition}

We denote the \textit{trace} of a matrix~$A$ as~$tr(A)$.

\begin{definition}[Node mapping]\label{def:mapping}
Given two graphs~$\CG_1$ and~$\CG_2$ of~$n$ nodes, a \textit{node mapping} between~$\CG_1$ and~$\CG_2$ is a bijection~$\pi: \CV_1 \rightarrow \CV_2$ where~$\forall v \in \CV_1, \pi(v)\in \CV_2$.
%; \textbf{(ii)} $\forall e=\langle v_1,v_2\rangle \in \CE_1, \langle \pi(v_1),\pi(v_2)\rangle\in \CE_2$.
\end{definition}
\vspace{-1mm}
\pke{Given graphs~$\mathcal{G}_{1}$ and~$\mathcal{G}_{2}$ with node counts~$n_1$ and~$n_2$, respectively, $n_1 \!<\! n_2$, we add~$(n_2 \!-\! n_1)$ isolated \textit{dummy} nodes with label~$\epsilon$ to~$\mathcal{G}_1$. Henceforward, we assume the two given graphs are of the same size.}

\begin{definition}[Graph Edit Distance under mapping $\pi$]\label{def:gedmap}
GED between $\CG_1$ and $\CG_2$ under $\pi$ is:
\vspace{-1mm}
\begin{multline}
GED_{\pi}(\CG_1,\CG_2) = \sum_{v \in \CV_1}d_v(\mathcal{L}(v),\mathcal{L}(\pi(v)))
+ \sum_{\substack{\langle v_1,v_2\rangle \in \CV_1\times \CV_1 \land\\ v_1 < v_2}}d_e(\langle v_1,v_2\rangle,\langle \pi(v_1),\pi(v_2)\rangle)
\end{multline}
where $d_v$ and $d_e$ are distance functions over the node labels and node pairs respectively. 
\end{definition}
\vspace{-1mm}
\pke{The distance between two identical node labels is~0.
%(See App.~\ref{app:assymetry} for arbitrary distance functions). 
If an existing edge is mapped to a non-existing edge, i.e, either~$\langle v_1, v_2\rangle \not \in \CE_1$ or~$\langle \pi(v_1), \pi(v_2)\rangle \not\in \CE_2$ the cost\footnote{We define it to be $\kappa^2$ instead of $\kappa$ since it eases the notational burden in subsequent derivations.} is~$\kappa^2$, otherwise~$0$. Intuitively, mapping from a dummy node/edge to a real one expresses insertion, while mapping from a real node/edge to a dummy one expresses deletion, and mapping from a real node to a real node of different label denotes replacement. Figure~\ref{fig:ged} in the appendix illustrates GED mappings with examples.}
%Similarly, we assume,  $d(\ell_1,\ell_2)$ models an \textit{insertion} if $\ell_1=\epsilon$, \textit{deletion} if $\ell_2=\epsilon$ and \textit{replacement} if $\ell_1\neq \ell_2$ and neither $\ell_1$ nor $\ell_2$ is a dummy.
%We assume $d$ to be a binary function, where $d(\ell_1,\ell_2)=1$ if $\ell_1\neq\ell_2$, otherwise, $0$. %Our framework easily extends to more general distance functions (details in App.~\ref{app:proofs}).

\begin{definition}[GED]\label{def:ged}
\pke{GED is the minimum distance among all mappings.}
\begin{equation}\label{eq:ged_form}
GED(\CG_1,\CG_2) = \min_{\forall \pi \in \Phi(\CG_1,\CG_2)}GED_{\pi}(\CG_1,\CG_2)
\end{equation} 
$\Phi(\CG_1, \CG_2)$ denotes all possible node maps from $\CG_1$ to $\CG_2$.
\end{definition}

\vspace{-0.1in}
\subsection{Mapping GED to Graph Alignment}
\vspace{-0.1in}
\pke{We now establish that unrestricted graph alignment (UGA)~\citep{survey1} forms an instance of GED. Building on this connection, we recast GED by \Cref{def:ged} as a \textit{generalized} graph alignment problem, leading to algebraic methods for GED estimation.}

\begin{definition}[Unrestricted Graph Alignment]
\pke{Unrestricted graph alignment calls to find a bijection $\pi: \CV_1\rightarrow \CV_2$ that minimizes edge disagreements between the two graphs. Formally:}
\begin{equation}\label{eq:alignment}
\min_{\pi \in \Phi(\CG_1,\CG_2)} \|AP_{\pi} - P_{\pi}B\|^{2}_{F},
\end{equation}
\pke{Here, $A$ and~$B$ are the adjacency matrices of graphs~$\mathcal{G}_1$ and~$\mathcal{G}_2$, respectively,~$\|.\|_{F}$ denotes the Frobenius Norm, and $P_\pi$ is a permutation matrix , where~$P_\pi[i,j] = 1$ if~$\pi(i) = j$, otherwise~$0$.}
\end{definition}
% \& Chemical Distance.}
% Graph alignment finds various applications, such as matching protein networks across species~\cite{biology1,biology2}, identifying users in social networks~\cite{socialnetworks}, and feature matching in computer vision~\cite{cv1,cv2}.%Notably, the expression in Equation~\ref{cd} is the square of \emph{Chemical Distance} (CD)~\cite{Kvasnicka1991ReactionAC} between~$\mathcal{G}_{1}$ and~$\mathcal{G}_{2}$}.
% \pk{~\cite{bento2018family} introduce the \emph{Modified Chemical Distance} (MCD) that incorporates node labels. Let~$\psi_{1}$ and~$\psi_{2}$ be mappings that associate nodes of graphs~$\mathcal{G}_{1}$ and~$\mathcal{G}_{2}$, both of size~$n$, to a metric space~$(\Omega, d)$. Formally, $\psi_{1}: [n] \rightarrow \Omega$ and $\psi_{2}: [n] \rightarrow \Omega$. The matrix~$D_{\psi_{1}, \psi_{2}}$ is:
% \begin{equation} \label{matrix_d}
% D_{\psi_{1}, \psi_{2}} = \left[ d(\psi_{1}(u), \psi_{2}(v)) \right]_{u, v \in [n]} \in \mathbb{R}^{n \times n}_{+},
% \end{equation}}
% The Modified Chemical Distance (MCD) is:
% \begin{equation}\label{mcd}
% \text{MCD}(\mathcal{G}_{1}, \mathcal{G}_{2}) = \min_{P \in \mathbb{P}^{n}} \|AP - PB\|_{F} + tr(P^{T}D_{\psi_{1}, \psi_{2}})
% \end{equation}
%The following theorem will be useful in the following.
%\subsection{Formulating (CD)$^2$ as an instance of GED}
%\subsubsection{Unrestricted Graph Alignment \& GED}

\pkl{The proof of the following theorem is in Appendix~\ref{app:proofs}.}

\begin{restatable}{theorem}{uga}\label{thm:equivalence}
\pke{Given graphs~$\mathcal{G}_{1}$ and~$\mathcal{G}_{2}$ of size~$n$, if the \pkn{\textit{edge insertion} and \textit{deletion} cost is~$\kappa^2=2$ and node substitution cost is~0}, then GED($\mathcal{G}_1$, $\mathcal{G}_2$) = $\min_{\pi \in \Phi(\CG_1,\CG_2)} \|AP_{\pi} - P_{\pi}B\|^{2}_{F}$.}
\end{restatable}
%By Lemma~\ref{lemma:ged}, ED($\pi$) = $\|AP - PB\|^{2}$. It is GED = $\min_{\pi}$ ED($\pi$) = $\min_{\pi \in \phi(\CG_1,\CG_2)} \|AP_{\pi} - P_{\pi}B\|^{2}_{F}$, since there is a one-to-one mapping between alignment functions~$\pi$ and permutation matrices~$\pi$%.; then GED($\mathcal{G}_1$, $\mathcal{G}_2$) = CD($\mathcal{G}_1$, $\mathcal{G}_2$)$^{2}$.}
%\end{proof}
%We show that by expressing the objective of Graph Alignment between two graphs~$\mathcal{G}_{1}$ and~$\mathcal{G}_{2}$, defined in Equation~\eqref{cd}, as an instance of GED. Consider an instance of GED between graphs~$\mathcal{G}_{1}$ and~$\mathcal{G}_{2}$ whose adjacency matrices are~$A, B$ respectively, 
% defined as follows:}
% \begin{itemize}
%     \item Node insertion cost = 0
%     \item Node deletion cost = 0
%     \item Node substitution cost = 0
%     \item Edge insertion cost = 2
%     \item Edge deletion cost = 2
%     \item Edge substitution cost = 0
% \end{itemize}
%\nips{We show that the GED between~$\mathcal{G}_{1}$ and~$\mathcal{G}_{2}$ by these edit costs is equal to the objective of Graph Alignment.}
%\begin{lemma}\label{lemma:ged}
%\pk{Given graphs~$\mathcal{G}_{1}$, $\mathcal{G}_{2}$ of size~$n$ and a node alignment function~$p: [n] \rightarrow [n]$ mapping each node~$i$ in~$\mathcal{G}_{1}$ to a node~$\pi(i)$ in~$\mathcal{G}_{2}$, the edit distance corresponding to the node alignment by the aforementioned edit costs is equal to~$\|AP - PB\|^{2}_{F}$, where~$\pi$ is a permutation matrix with~$P_{ij} = 1$ if~$\pi(i) = j$, otherwise~0.}
%\end{lemma}
\vspace{-3mm}
% \begin{corollary}
% Computing GED is \APXhard.
% \end{corollary}
% \vspace{-2mm}
% \begin{proof}
% \pke{Unrestricted Graph Alignment is \APXhard even with an approximation factor that grows linearly with~$n$~\cite{apx, grampa}. By Theorem~\ref{thm:equivalence}, it follows that GED computation is \APXhard.}
% \end{proof}
%\pki{The preceding lemmata establish that the objective of Unrestricted Graph Alignment is an instance of GED.} \nips{It is already established that  Having now established that Unrestricted Graph Alignment is an instance of GED, it follows that GED is APX-Hard to approximate even within an approximation factor that grows linearly with the number of nodes.}
\vspace{-0.1in}
\section{\name: Proposed Method}\label{sec:soln}
\vspace{-0.1in}
\pke{While \Cref{thm:equivalence} establishes graph alignment as a special case of GED, Equation~\eqref{eq:alignment} assumes a specific instance of edits costs and ignores node labels, setting node edit costs to~$0$. We next frame GED as a generalized graph alignment problem with \textit{arbitrary} edit costs.}

%\pk{Here we introduce our innovative unsupervised method that estimates the GED between a pair of graphs, $\mathcal{G}_{1}$ and~$\mathcal{G}_{2}$, and also generates a comprehensible node alignment corresponding to the estimated GED. We emphasize that the GED approximated by our method is explainable, as it returns an existing edit path from graph~$\mathcal{G}_{1}$ to~$\mathcal{G}_{2}$ whose edit cost corresponds to the approximated GED. As the true GED is the minimum edit cost over all possible node alignments, the returned GED upper-bounds the true GED.}
%\pk{Our approach builds upon the graph alignment. % in two stages:
% \begin{itemize}
% \item We show the objective of Unrestricted Graph Alignment, i.e., squared Chemical Distance, to be a GED instance. % by assigning costs to edit operations such as addition, deletion, and substitution of nodes and edges.
%We formulate an optimization for computing a Mildly-Constrained version of GED, \textsc{Mc-GED}, via the objective of graph alignment.
% \end{itemize}
%Since the defined optimization problem is computationally hard, we propose two approximations that minimize the function over the space of doubly stochastic matrices~$\mathbb{W}^{n}$ instead of the space of Permutation matrices~$\mathbb{P}^{n}$. After obtaining the solution from the approximations, we convert it into a permutation matrix that represents a valid node alignment by one of two methods, Greedy and Hungarian. Lastly, we output the generated permutation matrix defining the node alignment to which the approximated GED corresponds.}
\vspace{-0.1in}
\subsection{GED as Generalized Graph Alignment}

\pke{Given graphs~$\mathcal{G}_{1}$ and~$\mathcal{G}_{2}$, arbitrary costs for node edits, and cost~$\kappa^2$ for edge edits, where~$\kappa$ is a scalar, we propose a closed-form expression for \textit{generalized} graph alignment:}
\vspace{-2mm}
\begin{comment}
\begin{itemize}
\item Node insertion cost $c(\nodeins(j))$ = Node deletion cost $c(\nodedel(i))$, where $i \in \mathcal{V(G}_{1}),\; j \in \mathcal{V(G}_{2})$.
%, $c(\nodeins(j))$ is the associated insertion cost
%, $c(\nodedel(i))$ is the associated deletion cost
%\item Node substitution cost = $c(\nodesub(i, j))$, where $i \in \mathcal{V(G}_{1})$ and $j \in \mathcal{V(G}_{2})$.
%, $c(\nodesub(i, j))$ is the associated substitution cost
\item Edge insertion cost = $k^2$
\item Edge deletion cost = $k^2$
\item Edge substitution is an operation we do not allow.
\end{itemize}
where $c$ can be any arbitrary cost function and $k$ is a scalar. Under these mildly constrained edit costs, 
\end{comment}
\begin{equation} \label{ged_chdist}
\min_{\pi \in \Phi(\CG_1,\CG_2)} \frac{||\tilde{A}P_{\pi} - P_{\pi}\tilde{B}||_F^{2}}{2} + tr(P_{\pi}^{T}D)
\end{equation}

\pke{Let~$A, B$ be adjacency matrices of~$\mathcal{G}_{1}$, $\mathcal{G}_{2}$, respectively, having extended the smaller graph to the size of the larger by adding dummy nodes. We set~$\tilde{A} = \kappa \cdot A, \tilde{B} = \kappa \cdot B$ and define~$D$ as:} % what is this operation?

\begin{equation}
\label{eq:d}
d_{ij} =
\begin{cases}
d_{v}(\epsilon, \mathcal{L}(j)),    & \text{if } i \text{ is a dummy node in } \mathcal{G}_{1}\\
d_{v}(\mathcal{L}(i), \epsilon),    & \text{if } j \text{ is a dummy node in } \mathcal{G}_{2}\\
d_{v}(\mathcal{L}(i), \mathcal{L}(j))), & \text{if }\mathcal{L}(i)\neq\mathcal{L}(j)
\end{cases}
\end{equation}
\pke{where~$d_v$ is the distance function over the node labels by Definition~\ref{def:gedmap} and~$\epsilon$ is the label assigned to dummy nodes.}
%\begin{equation}
%\label{eq:d}
%d_{ij} =
%\begin{cases}
%\gamma,    & \text{if } i\in\mathcal{G}_{1}\text{ or }j\in\mathcal{G}_{2}\text{ is a dummy node} \\
%0, & \text{if }\mathcal{L}(i)=\mathcal{L}(j)\\
%\gamma, & \text{otherwise}% (node deletion followed by insertion)}
%\end{cases}
%\end{equation}
\pke{We show that, with~$\tilde{A}, \tilde{B}, D$ as above, Equation~\eqref{ged_chdist} amounts to GED with arbitrary edit costs. Intuitively, the first term captures edge edits under mapping~$\pi$, the second term node edits.} \pkl{The proof is in Appendix~\ref{app:proofs}.}

\begin{restatable}{theorem}{geduga}\label{thm:equivalence2}
Given two graphs~$\mathcal{G}_{1}$ and~$\mathcal{G}_{2}$ of size~$n$, GED($\mathcal{G}_1$, $\mathcal{G}_2$) = $\min_{\pi \in \Phi(\CG_1,\CG_2)} \frac{||\tilde{A}P_{\pi} - P_{\pi}\tilde{B}||_F^{2}}{2} + tr(P_{\pi}^{T}D)$, where $\tilde{A}, \tilde{B}$ and $D$ are defined as above.
\end{restatable}
\begin{comment}
\begin{lemma}\label{lemma:mild_ged}
\pk{Given two graphs~$\mathcal{G}_{1}$, $\mathcal{G}_{2}$ of size~$n$ whose adjacency matrices are~$A, B$ respectively, a constant~$k$ where~$\kappa^2$ is the edge insertion and deletion cost, matrix~$D$ as defined above and a permutation matrix~$P$, $\frac{||\tilde{A}P - P\tilde{B}||_F^{2}}{2} + tr(P^{T}D)$ equals the mildly-constrained edit distance (\textsc{Mc-ED}) by the above specified edit costs, for node alignment $p: [n] \rightarrow [n]$, where~$\pi(i) = j$ if and only if~$P_{ij} = 1$.}
\end{lemma}
\end{comment}

\begin{comment}
\begin{lemma}
\pk{For graphs~$\mathcal{G}_{1}$ and~$\mathcal{G}_{2}$ of size~$n$, the distance by Equation~\eqref{ged_chdist} equals the \textsc{Mc-GED} between~$\mathcal{G}_{1}$ and~$\mathcal{G}_{2}$.}
\end{lemma}

\begin{proof}
\pk{By Lemma~\ref{lemma:mild_ged}, $\frac{||\tilde{A}P - P\tilde{B}||_F^{2}}{2} + tr(P^{T}D)$ equals \textsc{Mc-ED}($p$), hence~$\min_{P \in \mathbb{P}^{n}} \frac{||\tilde{A}P - P\tilde{B}||_F^{2}}{2} + tr(P^{T}D)$ corresponds to~$\min_{p}$ \textsc{Mc-ED}($p$) due to the one-to-one mapping between~$p$ and~$P$. Thus, the distance computed by Equation~\eqref{ged_chdist} is the \textsc{Mc-GED} between~$\mathcal{G}_{1}$ and~$\mathcal{G}_{2}$.}
\end{proof}
\end{comment}

% \vspace{-3mm}
\pke{\ipfp~\cite{ipfp} also formulates GED as a quadratic assignment problem, yet it flattens the permutation matrix into a vector and operates on a cost matrix~$C = (c_{ik,jl})_{i,k,j,l}$, where~$c_{ik,jl}$ denotes the cost of editing edge~$(i, j)$ in one graph to edge~$(k, l)$ in the other. In contrast, \name preserves the permutation matrix structure and operates on adjacency matrices~$A$ and~$B$, expressing edge discrepancies through the difference of the permuted matrices~$\tilde{A}P$ and~$P\tilde{B}$. This structure-aware formulation reduces time complexity from~$O(n^4)$ in \ipfp to~$O(n^3)$ and is also more space-efficient: while~$C$ is a dense matrix of size~$n^2 \!\times\! n^2$, $\tilde{A}$ and~$\tilde{B}$ are~$n \!\times\! n$ and usually sparse. Moreover, \name is numerically more stable, while~$C$ becomes ill-conditioned and thus unsuitable for gradient-based optimization for similar~$A$ and~$B$, which render the rows and columns of~$C$ nearly linearly dependent. Besides, \name naturally accommodates permutation and doubly-stochastic constraints and maintains a spectral connection to the eigenvalues of~$A$ and~$B$, which enables the use of spectral techniques~\cite{grasp,inexact,isorank}. Lastly, \ipfp relies on off-the-shelf optimization methods, while \name uses a custom optimization strategy, which confers the advantages shown in \S~\ref{sec:exp}.}

\pkl{Grounded in our structure-aware reformulation of GED as generalized graph alignment problem based on adjacency matrices, we can leverage advances in graph alignment for GED estimation purposes. \fugal~\cite{fugal}, the current state-of-the-art solution for UGA, relaxes a quadratic assignment problem with an objective built on a non-convex correlation term to the feasible set of doubly stochastic matrices and applies the Frank--Wolfe algorithm~\cite{frank1956algorithm} guided by a Sinkhorn--Knopp normalization~\cite{cuturi2013sinkhorn} to iteratively step within that feasible set in a direction most aligned with the negative gradient. As our experimental study reveals, while this approach is good enough for graph alignment, where solutions are evaluated by the proportion of correctly aligned nodes, it yields poor results in terms of GED, where solutions are strictly evaluated by the difference of their GED cost from the ground truth. We conclude that GED estimation calls for a more rigorous approach directly targeting the convex GED cost as the core objective with stable gradient updates. Nonetheless, we adopt from \fugal the idea of refining a doubly stochastic matrix towards a quasi-permutation matrix.}
% \fugal cannot be directly applied to GED estimation, as demonstrated in \S: where?
\vspace{-0.1in}
%\vspace{-0.2in}
%\subsection{Optimization Strategy}\label{sec:optimizer} % no need
\vspace{-0.05in}
\subsection{Permutation-Inducing Regularization}
\vspace{-0.1in}
\pke{While Equation~\eqref{ged_chdist} provides a closed-form expression, finding the permutation matrix that minimizes it is notoriously hard, as the space of permutation matrices is not convex.
%Here, we relax this optimization problem to approximate the \textsc{Ged} by expanding the optimization domain from the discrete space of permutation matrices to the more tractable space of \emph{doubly stochastic} matrices (Recall Def.~\ref{def:matrix}), \pki{and then guide the solution towards a \emph {quasi-permutation matrix}}. %We propose two relaxations, \name-1 and~\name-2.}
To circumvent this non-tractability, we relax Equation~\eqref{eq:alignment} form the set of permutation matrices to that of doubly stochastic matrices~$\mathbb{W}^{n}$, rendering the problem convex~\citep{bento2018family}, and solve the relaxed form of Equation~\eqref{ged_chdist}:}
\vspace{-0.1in}
\begin{equation}
\label{relaxation1}
    \begin{gathered}
        \min_{\pk{P}\in\mathbb{W}^n} \frac{||\tilde{A}P - P\tilde{B}||_F^2}{2} + tr(P^{T}D)\\
        \text{Constraints: } P\textbf{1} = \textbf{1}, P^T\textbf{1} = \textbf{1}, 0 \leq P_{ij} \leq 1 
    \end{gathered}
\end{equation}

\pke{Equation~\ref{relaxation1} is convex, as it minimizes a convex function over a convex domain~\cite{boyd2004convex} and solvable with Adam~\cite{kingma14}, yet the optimal doubly-stochastic matrix does not solve our exact problem.} \pkl{Still, these two matrix domains are connected as follows~\cite{fugal}; the proofs are in Appendix~\ref{app:proofs}.}

\begin{restatable}{lemma}{doublystoch}\label{theorem:doubly-stoch}
A doubly-stochastic matrix~$A$ with~$tr(A^{T}(J - A)) = 0$ is a permutation matrix.
\end{restatable}

\pke{Utilizing this connection, we add a bias to our objective function in the following form.}
%\pki{yet restrict the solution to a \emph{quasi-permutation} matrix.} By Theorem~\ref{theorem:doubly-stoch}, if~$P$ is doubly-stochastic and~$tr(P^{T}(J - P)$ = 0, then~$P$ is a permutation matrix. We rewrite the expression to find mildly-constrained GED in Eq.~\eqref{ged_chdist} as:
\begin{comment}
\begin{equation} \label{new_ged}
    \begin{gathered}
        \min_{P \in \mathbb{W}^{n}} \frac{||AP - PB||_F^{2}}{2} + tr(P^{T}D) \\
        \text{Constraints: } tr(P^{T}(J - P)) = 0
    \end{gathered}
\end{equation}
We introduce parameters~$\mu$ and~$\lambda$ and relax further to:
\end{comment}
%\begin{equation} \label{relaxation2}
%    \min_{P \in \mathbb{W}^{n}} \frac{||AP - PB||_F^2}{2} + \mu \cdot (tr(P^{T}D)) + \lambda \cdot (tr(P^{T}(J - P)))
%\end{equation}
%Equivalently:
\vspace{-0.1in}
\begin{equation}\label{relaxation2_v2}
    \begin{gathered}
        \min_{\pk{P}} \!\frac{||\tilde{A}\!P \!-\! P\!\tilde{B}||_F^2}{2} \!+\! \mu \cdot (tr(P^{T}\!D)) + \lambda \cdot (tr(P^{T}\!(J \!-\! P)))\\
        \text{Constraints: } P\textbf{1} \!=\! P^T\textbf{1} \!=\! \textbf{1}, 0 \!\leq\! P_{ij} \!\leq\! 1
    \end{gathered}
\end{equation}
where $\mu$ and $\lambda$ are weight parameters. \pkl{\fugal extracts a non-convex correlation term from this objective; contrarily, we preserve convexity and thus derive a spectral guarantee:}

\begin{restatable}{theorem}{convexity}\label{theorem:convexity}
The function in Equation~\eqref{relaxation2_v2} is convex for $\lambda \!\leq\! \frac{(\lambda_{i}(\tilde{A}) - \lambda_{j}(\tilde{B}))^2}{2}$,
for all \(i, j \!\in\! \{1, 2, \dots, n\}\), where~\(\lambda_{i}(\tilde{A})\) and~\(\lambda_{j}(\tilde{B})\) represent the eigenvalues of~\(\tilde{A}\) and~\(\tilde{B}\), respectively.
\end{restatable}

\pke{For~$\lambda$ = 0, the problem in Equation~\eqref{relaxation2_v2} is convex. To derive a quasi-permutation matrix, we solve Equation~\eqref{relaxation2_v2} with~$\lambda = 0$ using Adam and refine the solution by gradually increasing~$\lambda$, until it diverges.} \pkl{This regularizer, which drives the double-stochastic matrix to a permutation matrix drastically enhances approximation accuracy, as we show in Appendix~\ref{app:ablation}.}
\vspace{-0.1in}
\subsection{\madam Details}\label{sec:rounding}
\vspace{-0.1in}
\renewcommand{\algorithmicrequire}{\textbf{Input:}}
\renewcommand{\algorithmicensure}{\textbf{Output:}}
\begin{wrapfigure}{r}{0.57\textwidth}
\vspace{-16mm}
\begin{minipage}{0.57\textwidth}
\begin{algorithm}[H] % it looks unaligned at the top
\caption{\textsc{M-Adam}}\label{M-Adam}
\begin{flushleft} 
{\scriptsize
\textbf{Notations:}\\
$f = \frac{||\tilde{A}P - P\tilde{B}||_F^2}{2} + \mu \cdot (tr(P^{T}D)), \quad
g = tr(P^{T}(J - P))$\\
$pnlt \!=\! ||P\textbf{1} \!-\! \textbf{1}||^{2} \!+\! ||P^T\textbf{1} \!-\! \textbf{1}||^{2} \!+\! ||max(0, -P)||^{2} \!+\! ||max(0, P \!-\! J)||^{2}$\\
%$\sigma$ is the penalty-coefficient\\
\textbf{Input:} matrices~$\tilde{A}$, $\tilde{B}$, $D$ \textbf{Output:} permutation matrix $P$\\
\textbf{Algorithm}:}
\end{flushleft}
\begin{algorithmic}[1] %[1] enables line numbers
{\scriptsize
\STATE $P \!\gets\! I$, $\sigma \!\gets\! 5$,$\lambda \!\gets\! 0$, $m_{0} \!\gets\! 0$, $v_{0} \!\gets\! 0$, $\beta_1 \!\gets\! 0.9$, $\beta_2 \!\gets\! 0.99$, $\tilde{P} \!\gets\! I$, $\tilde{H} \gets I$ \label{lin:init}
\WHILE {$\textit{true}$}
\STATE $t \gets 0$
\WHILE{$\textit{not converged}$}
\STATE $t \gets t + 1$
\STATE $\textit{grad} \gets \nabla f + \sigma \cdot \nabla pnlt + \lambda \cdot \nabla g $ \label{lin:gradient_start}
\STATE $m_{t} \gets \beta_{1} \cdot m_{t - 1} + (1 - \beta_{1}) \cdot grad$; $\hat{m_{t}} \gets m_{t}/(1 - \beta_{1}^{t})$
\STATE $v_{t} \gets \beta_{2} \cdot v_{t - 1} + (1 - \beta_{2}) \cdot grad^{2}$; $\hat{v_{t}} \gets v_{t}/(1 - \beta_{2}^{t})$
\STATE $P \gets P - \alpha \cdot \hat{m_{t}}/(\sqrt{\hat{v_{t}}} + \epsilon)$ \label{lin:grad_end}
\ENDWHILE
\IF {$\textit{diverged}$}
\STATE $\textit{break}$
\ENDIF
\STATE $\sigma \gets \sigma*2, \quad \lambda \gets \lambda + 0.5$ \label{lin:inc_lambda}
\STATE $H \gets \text{Hungarian(}P)$, $\tilde{A} \gets H\tilde{A}H^\top$, $D \gets HD$
\STATE $\tilde{P} \gets \tilde{H}^\top P$, $\tilde{H} \gets H\tilde{H}$
\IF {$\sigma > \sigma_{th}$} 
\STATE $\textit{break}$
\ENDIF
\ENDWHILE
\RETURN \pke{$\tilde{P}$} \label{lin:return}}
\end{algorithmic}
\end{algorithm}
\end{minipage}
\vspace{-2mm}
\end{wrapfigure}
\Cref{M-Adam} \pkl{outlines our Modified Adam (\madam) algorithm, which initializes~$P$ as an identity matrix and~$\lambda$ to~$0$ (Line~\ref{lin:init}), and gradually increases~$\lambda$ (Line~\ref{lin:inc_lambda}). For each~$\lambda$, it starts from the solution of the previous round and iteratively updates it using the objective's gradient (Lines~\ref{lin:gradient_start}--\ref{lin:grad_end}). We employ the \textit{penalty method}~\citep{yeniay2005penalty} to enforce doubly-stochastic matrix constraints. \iclr{For a given value of~$\lambda$, the relaxed solution~$P$ is rounded to a permutation matrix~$H$ via Hungarian, which is then used to transform the problem in the subsequent iteration (see \S~\ref{sec:invlabel}).} Figure~\ref{fig:toy_example} illustrates the process with an example. \madam outputs a permutation matrix that yields an edit path for the approximated GED~\cite{hungarian}. As the true GED is the least edit cost over all alignments, the returned GED upper-bounds the true GED. Moreover, \madam is a \emph{deterministic} algorithm; for any given pair of input matrices, it always returns the same output.}
\vspace{-0.1in}
\subsection{Inverse Relabeling}\label{sec:invlabel}
\vspace{-0.1in}
\iclr{Here, we propose an \emph{inverse relabeling} strategy in \madam. The core term of our objective is~$||\tilde{A} -P\tilde{B}P^{T}||_F^2$, to be minimized over~$\mathbb{W}^n$. After the first gradient-based update iteration with fixed~$\lambda$ (outer loop in \madam), we begin enforcing permutation constraints via a regularizer. Let $H$ denote a permutation matrix obtained by rounding the relaxed solution $P$ using Hungarian projection.

Since the feasible set $\mathbb{P}^n$ is discrete, gradients are computed in the relaxed domain $\mathbb{W}^n$. However, continuing the optimization near a non-identity permutation $H$ is inefficient. A non-identity $H$ acts as a rotation of the problem's coordinate system, causing the components of the gradient to become highly coupled. This motivates recentering the problem after each outer iteration. Specifically, we transform $\tilde{A} \leftarrow H \tilde{A} H^\top$. This transformation is equivalent to the variable change $\tilde{P} = H^\top P$, as: \\
\[
\|\tilde{A} - P \tilde{B} P^\top \|_F^2 \rightarrow \| H \tilde{A} H^\top - P \tilde{B} P^\top \|_F^2 = \| \tilde{A} - H^\top P \tilde{B} \tilde{P}^\top H \|_F^2 = \| \tilde{A} - \tilde{P} \tilde{B} \tilde{P}^\top \|_F^2,
\]
\pkl{This variable change to~$\tilde{P}$ and multiplication by~$H^\top$ revokes the permutation, or \emph{inverts the labeling}, introduced by~$H$, without altering the feasible space: $\tilde{P}\! \in \!\mathbb{W}^n \iff P = H \tilde{P} \in \mathbb{W}^n$, since multiplying a doubly stochastic matrix by a permutation matrix preserves row and column sums and non-negativity.} \pkl{The updated~$\tilde{P}$ satisfies~$\tilde{P} \approx H^\top H = I$, hence gradient updates are performed in a coordinate system centered around the identity matrix $I$, allowing for more efficient and accurate corrections to small errors. Our ablation study in \S~\ref{app:ablation} validates the effectiveness of this transformation.}}

\vspace{-0.1in}
\section{Experiments}\label{sec:exp}
\vspace{-0.1in}
Here, \pke{we present a comprehensive evaluation of \name, addressing the following aspects:}
\begin{itemize}
% \item \textbf{Efficacy:} \name consistently outperformed state-of-the-art neural and non-neural approaches across various datasets under different cost settings achieving state-of-the-art performance. 
% % Despite supervised neural solvers benefiting from the advantage of training data, \name consistently ranks as the top performer. This remarkable performance demonstrates that the absence of NP-hard ground truth data and the elimination of dataset-specific model training are not impediments to achieving highly accurate Graph Edit Distance (GED) approximations, showcasing the significant advantages of our approach.
% \item \textbf{Scalability:} \name scales significantly well to large graphs and consistently outperforms the baselines.
% \item \textbf{Efficiency and greenhouse implications:} \name incurs lower computation cost than unsupervised methods of similar performance. \name operates on CPU while neural methods require GPU access raising concerns for environmental impact.
% \end{itemize}
\vspace{-0.1in}
\item \textbf{Efficacy:} \pkl{\name tops supervised and heuristic methods across datasets and costs.}
\item \textbf{Scalability:} \pkl{\name scales well to large graphs, consistently surpassing baselines.}
\item \textbf{Efficiency:} \pkl{\name incurs lower computational costs than heuristic methods with better performance; as it runs on CPUs, it curtails carbon emissions.}
\end{itemize}
\vspace{-0.2in}
\subsection{Experimental Setup}\label{sec:setup}

Appendix~\ref{app:setup} \pke{outlines the hardware and software\footnote{Our C++ code and datasets are at \url{https://anonymous.4open.science/r/Eugene-1107/}} environment, Appendices~\ref{app:parameters} presents the parameters used, and Appendix~\ref{app:ablation} reports on an ablation study.}

\textbf{Baselines:} We compare \name to~\pke{15} state-of-the-art supervised and optimization based heuristic methods. These include the following supervised methods: \graphedx~\cite{graphedx}, \gmn~\cite{icmlged}, \greed~\cite{ranjan&al22}, \eric~\cite{eric}, \simgnn~\cite{simgnn}, \hmn~\cite{h2mn}, \egsc~\cite{egsc}, \gotsim~\cite{graphotsim}, ~\gedgnn~\cite{gedgnn}, ~\gmsm~\cite{gmsm}.
%To compute the GED, \gmn and \greed use the Euclidean distance between the vector representations of two graphs. \hmn is an early interaction network that leverages higher-order node similarity through hypergraphs. \eric, \simgnn, and \egsc utilize neural networks to calculate the distance between two graphs. Additionally, these three methods predict a score based on the normalized GED in the form $\exp\left(-\frac{2 \text{GED}(G, G')}{|V| + |V'|}\right)$. Hence, we rescale the predicted score to obtain the GED prediction as $\text{GED}(G, G') = -\frac{(|V| + |V'|) \log(s)}{2}$. 
We exclude the neural approximation algorithms \graphsim~\cite{graphsim} as \graphedx and \hmn have shown vastly better performance~\cite{graphedx,h2mn}. \gennastar~\cite{WangCVPR21} does not scale for graphs of sizes more than $10$, hence excluded from the analysis. \pke{Among the neural methods included, ~\gedgnn, ~\gmsm and~\gotsim provide a node mapping corresponding to the estimated GED.} \iclr{With all baselines, when edit costs are uniform, we use the official author-released codebases with the original training protocols and default hyperparameters. However, existing baselines do not support non-uniform edit costs, except for \graphedx, which extended support to non-uniform costs and released adapted codebases for all baselines. In the non-uniform cost setting, we use these fine-tuned and publicly available versions provided by the \graphedx authors.}
%Source code of neural baselines is available at \url{https://github.com/structlearning/GraphEdX}.

In the heuristic methods category, we compare with the five best-performing methods from the benchmarking study by~\cite{blumenthal2019gedlib}, namely,  \btight~\cite{btight}, \lpged~\cite{LEROUGE2017254}, \adjip~\cite{adjip}, \ipfp~\cite{ipfp} and \compact~\cite{blumenthal2020exact}. %The methods \lpged, \adjip, and \compact employ a mixed integer programming framework based on the LP-GED paradigm to approximate the GED. In contrast, \btight iteratively solves instances of the linear sum assignment problem or the minimum-cost perfect bipartite matching problem. Other methods, such as Branch and Node, were excluded from the analysis despite their faster execution times, as they have been shown to perform poorly compared to LP-GED based approaches and \btight~\cite{Blumenthal}. 
All these heuristic methods furnish an edit path that corresponds to the approximated GED. We utilized the GEDLIB~\cite{blumenthal2019gedlib} implementation of these methods in our evaluations.

\textbf{Datasets:} \pkl{Table~\ref{tab:datasets} lists the datasets we use. App.~\ref{app:datasets} discusses the semantics. \aids, \molhiv, \mutag, \code are labeled whereas \imdb, \coil, \triangles, Netscience and HighSchool are unlabeled.}

\begin{wraptable}{r}{0.52\textwidth}
\scriptsize
\vspace{-5mm}
\caption{Datasets.}\label{tab:datasets}
\vspace{-2mm}
    \centering
    \scalebox{1}{
        \begin{tabular}{lrrrrl}
        \toprule
        \textbf{Name} & \textbf{Avg $|\mathcal{V}|$} & \textbf{Avg   $|\mathcal{E}|$} & \textbf{\# labels} & \textbf{Domain}\\
        \midrule
        \aids & 11.83 & 24.14 & 38 & Biology \\
        \molhiv & 15.47 & 31.86 & 119 & Biology\\
        \mutag & 23.32 & 44.64 & 14 & Biology\\
        \code & 18.61 & 37.42 & 97 & Software\\
        \imdb & 11.49 & 63.74 & - & Movies\\
        \coil & 8.70 & 34.44 & - & Vision\\
        \triangles & 9.11 & 20.16 & - & Synthetic\\
        Netscience & 379 & 914 & - & Collaboration \\
        HighSchool & 327 & 5818 & - & Proximity\\
        \bottomrule
        \end{tabular}}
\vspace{-4mm}
\end{wraptable}
\textbf{Train-Val-Test Splits:} \pkl{As in~\citep{graphedx}, we remove isomorphic graphs from the datasets prior to training neural methods to mitigate isomorphism bias via leakage between training and testing~\citet{isobias}. Further, for each dataset, we restrict to the  graphs of size less than~$25$ to ensure feasibility of ground truth GED computation. \pke{As in~\citep{ranjan&al22} and~\citep{graphedx}, we used MIP-F2~\citep{LEROUGE2017254} with a time limit of~600 seconds for each graph pair and kept pairs that yielded equal lower and upper bounds as ground truth GED.} The training set consists of~5$k$ randomly sampled graph pairs, while the validation and test sets each consist of 1$k$ randomly sampled pairs each.}
%Only pairs with equal lower and upper bounds are included in the evaluation.

\textbf{Cost Settings:} \pke{We evaluate the performance under three different edit cost settings:}
\vspace{-0.1in}
\begin{itemize}
\item \textbf{Case 1 (Nonuniform costs):} The node insertion cost is 3, node deletion cost is 1, edge insertion and deletion costs are 2, and the node substitution cost is 0.
% \vspace{-0.05in}
\item \textbf{Case 2 (Nonuniform costs with substitution):} \pkl{In addition to Case~1, substituting nodes with unequal labels incurs cost. If the substituted node label is the nearest neighbor based on the similarity ranking of node labels, the cost is~1, otherwise~2. As an illustrative case, the distance between labels is taken as the difference between their label IDs.}
% \vspace{-0.05in}
\item \textbf{Case 3 (Uniform costs):} \pke{Node/edge insertion and deletion costs~1, node substitution~0.}
\end{itemize}
\vspace{-0.05in}
\iclr{Cost Settings 1 and 3 closely follow those proposed in \graphedx. We introduce Cost Setting 2 to further increase the difficulty of the task. Unlike the other settings, the cost of an edit operation in this case is non-static, it dynamically varies based on the node labels involved, thereby requiring models to account for contextual variations during alignment. We also evaluate on edits costs inspired from chemistry. The results are discussed in App.~\ref{app:domaincosts}.}

\textbf{Metrics:} We use two metrics to assess GED approximation and interpretability:
\begin{enumerate*}[label=(\roman*)]
    \item Mean Absolute Error (MAE), and
    \item Strict Interpretability (SI).
\end{enumerate*} MAE serves as a metric to quantify the closeness of the predicted GED to the true GED. SI is measured as the fraction of graph pairs for which the predicted GED matches the true GED. A match between the predicted and true GED indicates that the alignment produced by the method is optimal. Consequently, SI reflects the algorithm's ability to produce the optimal node mapping and serves as a measure of interpretability.
\vspace{-0.1in}
\subsection{Benchmarking Accuracy (MAE)}\label{exp:accuracy}
\vspace{-0.1in}
\pke{Table~\ref{tab:accuracy_static_non_uniform} presents approximation accuracy in terms of MAE on benchmark datasets under the non-uniform cost setting (Case~1) and the non-uniform cost with substitution setting (Case~2). Appendix~\ref{app:acc_uniform} shows the comparison under the uniform cost setting and Appendix~\ref{app:unlabelled} shows that on unlabeled datasets. In all cases, \name outperforms all baselines.}

\textbf{Comparison with Supervised Baselines:} \pkl{\name outperforms all supervised baselines---including those providing node alignments---across datasets and cost settings by a large margin. Under the nonuniform cost setting, it achieves up to~44\% lower MAE on \code and a~72\% reduction on \aids compared to the next best method. For nonuniform costs with substitution, the improvement margin ranges from~44\% on \mutag to~63\% on \molhiv. \graphedx, \egsc, and \eric demonstrate the second-best performance.}

\textbf{Comparison with Heuristic Baselines:} \name demonstrates a substantial improvement over heuristic baselines. The margin of improvement exceeds 80\% across all datasets and both cost settings when compared to the next-best method, \adjip. Methods \btight and \compact perform considerably worse than \name.

\pkl{Table~\ref{tab:accuracy_static_non_uniform} further reveals that heuristic baselines fall short of supervised ones, which explains why the community shifted to supervised methods, despite their lack of interpretability, poor generalizability, and costly training. Though heuristic, \name tops supervised baselines and grants interpretability. Contrarily, supervised methods that yield node alignments tend to lag, as they trade accuracy for interpretability. \name makes no such compromise.}

\definecolor{1st}{rgb}{0.8, 1, 0.5}
\definecolor{2nd}{rgb}{1.0, 0.9, 0.3}
\definecolor{3rd}{rgb}{0.9,1,0.7}

\begin{table*}[!h] % we need !h for compression purposes
\vspace{-2mm}
\caption{\pkl{Accuracy comparison among baselines in MAE under different cost settings; green and yellow cells denote the best and second-best performance, respectively, for each dataset.}}\label{tab:accuracy_static_non_uniform}
\vspace{-2mm}
% \caption{RMSE on large datasets.}\label{tab:accuracy2}
% \vspace{-2mm}
\centering
\scalebox{0.85}{
    \begin{tabular}{p{2.5cm}|p{1.2cm}p{1.2cm}p{1.2cm}p{1.2cm}|p{1.2cm}p{1.2cm}p{1.2cm}p{1.2cm}} 
      \toprule
      & \multicolumn{4}{c|}{Cost Setting Case 1} & \multicolumn{4}{c}{Cost Setting Case 2}\\
      \textbf{Methods} & \aids & \molhiv & \code & \mutag & \aids & \molhiv & \code & \mutag\\
      \midrule
      \eric & \cellcolor{2nd}{1.17} & 1.38 & 1.48 & 4.80 & \cellcolor{2nd}{1.25} & \cellcolor{2nd}{1.59} & 1.71 & 1.89\\
      \egsc & 1.35 & 1.58 & 1.65 & \cellcolor{2nd}{1.59} & 1.35 & 1.71 & 1.79 & \cellcolor{2nd}{1.80}\\
      \graphedx & 1.54 & \cellcolor{2nd}{1.36} & \cellcolor{2nd}{1.33} & 2.39 & 2.06 & 2.10 & \cellcolor{2nd}{1.56} & 2.80\\
      \hmn & 1.53 & 2.00 & 1.90 & 1.74 & 1.58 & 2.08 & 2.34 & 2.00\\
      \gmn & 3.35 & 5.25 & 2.68 & 5.52 & 3.64 & 5.83 & 2.67 & 6.34\\
      \greed & 2.98 & 5.03 & 2.48 & 5.12 & 3.39 & 5.36 & 2.62 & 5.32\\
      \simgnn & 1.55 & 1.98 & 1.85 & 1.91 & 1.70 & 2.09 & 2.01 & 2.49\\
      \midrule
      \gedgnn & 2.37 & 4.23 & 2.61 & 2.46 & 2.28 & 3.60 & 3.36 & 3.86\\
      \gotsim & 7.53 & 14.49 & 8.15 & 10.89 & 10.66 & 22.19 & 12.07 & 15.38\\
      \gmsm & 15.04 & 25.57 & 21.16 & 26.81 & 21.08 & 34.12 & 32.49 & 35.59\\
      \midrule
      \btight & 7.97 & 9.86 & 13.91 & 15.02 & 6.95 & 9.95 & 21.47 & 13.62\\
      \adjip & 1.69 & 4.06 & 5.05 & 4.30 & 3.58 & 5.97 & 6.70 & 6.85\\
      \lpged & 5.41 & 10.63 & 6.28 & 10.64 & 5.8 & 13.47 & 11.08 & 13.82\\
      \compact & 2.95 & 7.21 & 8.39 & 7.13 & 6.18 & 10.29 & 12.72 & 10.78\\
      \ipfp & 5.63 & 9.99 & 6.39 & 9.53 & 8.47 & 14.27 & 13.43 & 14.36\\
      \midrule
      \name & \cellcolor{1st}{0.33} & \cellcolor{1st}{0.65} & \cellcolor{1st}{0.75} & \cellcolor{1st}{0.68} & \cellcolor{1st}{0.58} & \cellcolor{1st}{0.79} & \cellcolor{1st}{0.58} & \cellcolor{1st}{1.01}\\
      \bottomrule
    \end{tabular}}
\vspace{-3mm}
\end{table*}

\textbf{Unlabeled datasets:} \pke{We observed a similar trend on unlabeled data, as shown in App~\ref{app:unlabelled}, \name achieving an even greater margin of improvement. That is expected, as the absence of node features limits the effectiveness of GNN-based methods, which distinguish nodes by features. We note the highest improvement with IMDB dataset, which is also the densest. High density causes oversquashing in GNNs~\cite{oversquashing}, and is a likely reason for subpar performance of neural models.}
\vspace{-0.2in}
\subsection{Accuracy (SI)}
\vspace{-0.1in}
Table~\ref{tab:si} presents the comparison of \name with other baselines in terms of the Strict Interpretability (SI) metric. While few neural baselines do not explicitly provide alignments, we found the SI score for all supervised methods to be~0 across all cost settings. \pke{This finding indicates that, albeit some neural methods provide explicit node alignments, they fall short in alignment quality.} We thus omit these scores from the table. \name consistently achieves higher~SI scores compared to other heuristic methods, with an improvement of up to~69\% on the \code dataset under cost setting Case~1. These superior~SI scores highlight \name's ability to deliver optimal node alignments. Although supervised baselines generally provide better GED approximations than heuristic methods, heuristic baselines offer better interpretability. \name surpasses all baselines in both approximation accuracy and interpretability metrics, establishing itself as the new state-of-the-art for GED approximation while maintaining interpretability of the approximated GED.

\definecolor{1st}{rgb}{0.8, 1, 0.5}
\definecolor{2nd}{rgb}{1.0, 0.9, 0.3}
\definecolor{3rd}{rgb}{0.9,1,0.7}

\begin{table*}[!h] % we need !h for compression purposes
\vspace{-3mm}
\caption{\pkl{Accuracy comparison in terms of SI under different cost settings; green and yellow cells denote the best and second-best performance, respectively, for each dataset.}}\label{tab:si}
\vspace{-2mm}
% \caption{RMSE on large datasets.}\label{tab:accuracy2}
% \vspace{-2mm}
\centering
\scalebox{0.8}{
    \begin{tabular}{p{2.5cm}|p{1.2cm}p{1.2cm}p{1.2cm}p{1.2cm}|p{1.2cm}p{1.2cm}p{1.2cm}p{1.2cm}} 
      \toprule
      & \multicolumn{4}{c|}{Cost Setting Case 1} & \multicolumn{4}{c}{Cost Setting Case 2}\\
      \textbf{Methods} & \aids & \molhiv & \code & \mutag & \aids & \molhiv & \code & \mutag\\
      \midrule
      \btight & 0.02 & 0.02 & 0.01 & 0.01 & 0.01 & 0.01 & 0.01 & 0.01\\
      \adjip & \cellcolor{2nd}{0.90} & \cellcolor{2nd}{0.69} & \cellcolor{2nd}{0.48} & \cellcolor{2nd}{0.62} & \cellcolor{2nd}{0.69} & \cellcolor{2nd}{0.65} & \cellcolor{2nd}{0.63} & \cellcolor{2nd}{0.46}\\
      \lpged & 0.44 & 0.10 & 0.05 & 0.04 & 0.57 & 0.15 & 0.03 & 0.07\\
      \compact & 0.72 & 0.31 & 0.03 & 0.20 & 0.46 & 0.31 & 0.16 & 0.24\\
      \ipfp & 0.04 & 0.02 & 0.03 & 0.02 & 0.01 & 0.01 & 0.01 & 0.01\\ 
      \midrule
      \name & \cellcolor{1st}{0.91} & \cellcolor{1st}{0.84} & \cellcolor{1st}{0.82} & \cellcolor{1st}{0.83} & \cellcolor{1st}{0.71} & \cellcolor{1st}{0.67} & \cellcolor{1st}{0.74} & \cellcolor{1st}{0.59}\\
      \bottomrule
    \end{tabular}}
\vspace{-4mm}
\end{table*}
%\vspace{-0.1in}
\subsection{Accuracy on Large Graphs}\label{sec:large_acc}
\vspace{-0.1in}
\pkl{The complexity of GED estimation rises with graph size due to the exponential growth of mappings in combinatorial space. We evaluate performance exclusively on large graphs to explicitly investigate this aspect of scalability. 
%Since the acquisition of training data for GED learning is resource-intensive due to the NP-hardness of GED computations, supervised approaches typically need train on smaller graphs and apply the acquired knowledge to larger graphs of previously unseen sizes. 
We consider graphs with sizes in the range~$[25, 50]$ in the test split. Table~\ref{tab:mae_unseen} presents the MAE results under Case~1 and Case~2 cost settings, which demonstrate the superior scalability of \name to large graphs, with up to~66\% lower MAE than the next best performer, \hmn. Other methods exhibit significantly higher MAE. These findings underscore the practical applicability of \name for GED approximation on large graphs. SI comparison on large graphs appears in Appendix~\ref{app:si_unseen}.}

\definecolor{1st}{rgb}{0.8, 1, 0.5}
\definecolor{2nd}{rgb}{1.0, 0.9, 0.3}
\definecolor{3rd}{rgb}{0.9,1,0.7}

\begin{table*}[!h] % we need !h for compression purposes
\vspace{-2mm}
\caption{\pkl{Accuracy among baselines in MAE under different cost settings; graph sizes in~$[25,50]$; green and yellow cells denote best and second-best performance, respectively.}}\label{tab:mae_unseen}
\vspace{-2mm}
% \caption{RMSE on large datasets.}\label{tab:accuracy2}
% \vspace{-2mm}
\centering
\scalebox{0.85}{
    \begin{tabular}{p{2.5cm}|p{1.2cm}p{1.2cm}p{1.2cm}p{1.2cm}|p{1.2cm}p{1.2cm}p{1.2cm}p{1.2cm}} 
      \toprule
      & \multicolumn{4}{c|}{Cost Setting Case 1} & \multicolumn{4}{c}{Cost Setting Case 2}\\
      \textbf{Methods} & \aids & \molhiv & \code & \mutag & \aids & \molhiv & \code & \mutag\\
      \midrule
      \eric & 19.70 & 9.08 & 12.24 & 14.64 & 18.46 & 14.08 & 29.14 & 9.47\\
      \egsc & 35.68 & 12.68 & 15.02 & 15.12 & 30.22 & 16.92 & 16.04 & 14.31\\
      \graphedx & 24.44 & 21.65 & 33.01 & 21.82 & 20.75 & 17.01 & 34.01 & 15.98 \\
      \hmn & \cellcolor{2nd}{6.48} & \cellcolor{2nd}{4.59} & \cellcolor{2nd}{5.70} & \cellcolor{2nd}{3.44} & 10.86 & \cellcolor{2nd}{5.15} & \cellcolor{2nd}{10.42} & \cellcolor{1st}{4.54}\\
      \gmn & 9.60 & 10.82 & 8.52 & 9.80 & 9.99 & 13.68 & 14.57 & 11.03\\
      \greed & 10.05 & 10.20 & 8.46 & 9.28 & \cellcolor{2nd}{9.66} & 9.50 & 12.09 & 9.92\\
      \simgnn & 28.77 & 10.58 & 14.02 & 7.52 & 25.61 & 12.63 & 50.51 & 12.70\\
      \midrule
      \gedgnn & 25.78 & 11.83 & 36.75 & 19.96 & 23.29 & 15.27 & 25.17 & 17.18\\
      \gotsim & 29.03 & 25.93 & 26.87 & 24.62 & 29.78 & 32.47 & 31.58 & 30.48\\
      \gmsm & 44.66 & 44.62 & 49.65 & 44.22 & 21.08 & 50.90 & 66.06 & 55.94\\
      \midrule
      \btight & 29.76 & 24.95 & 31.54 & 27.86 & 26.62 & 23.23 & 26.27 & 28.72\\
      \adjip & 23.00 & 21.98 & 34.52 & 21.54 & 17.81 & 11.95 & 46.42 & 17.00\\
      \lpged & 23.22 & 11.19 & 21.92 & 15.05 & 30.32 & 11.56 & 42.86 & 17.95\\
      \compact & 73.30 & 40.02 & 76.71 & 56.84 & 59.33 & 28.95 & 47.20 & 41.18\\
      \ipfp & 17.86 & 14.65 & 16.51 & 16.48 & 18.65 & 18.47 & 24.88 & 20.16 \\
      \midrule
      \name & \cellcolor{1st}{4.45} & \cellcolor{1st}{3.88} & \cellcolor{1st}{4.14} & \cellcolor{1st}{2.80} & \cellcolor{1st}{3.25} & \cellcolor{1st}{3.73} & \cellcolor{1st}{4.33} & \cellcolor{2nd}{4.74}\\
      \bottomrule
    \end{tabular}}
\vspace{-2mm}
\end{table*}

\pkl{Figure~\ref{fig:heatmap_code2} presents MAE heatmaps on \code for cost setting Case~1. Each point stands for a graph pair~$\mathcal{G}_{Q}$, $\mathcal{G}_{T}$ with coordinates~(GED($\mathcal{G}_{Q}$, $\mathcal{G}_{T}$), ($|\mathcal{V}_Q|$ + $|\mathcal{V}_T|$)/2). Heatmaps for \egsc, \hmn, and \graphedx have a discernibly darker hue,
%Moreover, darker shades are more pronounced in the bottom row where we focus on graphs with more than 25 nodes. %upper and right portions, implying a deterioration in performance as query sizes and GED values rise. Consequently, 
corroborating that \name enjoys better scalability in graph size and GED value. Appendix~\ref{app:heatmaps} shows heatmaps for other datasets, while \pke{Appendix~\ref{app:vlarge_graphs} presents results on two thousand-scale collaboration networks, Netscience~\citep{netscience} and HighSchool~\citep{highschool}. To our knowledge, no prior GED estimation method handles graphs of this scale.}}

\begin{figure*}[!h]
\vspace{-4mm}
\centering
%\subfloat[\name]{\includegraphics[width =1.8in]{figures_heatmaps/data_static_non_uniform/Eugene_ogbg-code2_heatmap_25.pdf}}
%\subfloat[\egsc]{\includegraphics[width=1.8in]{figures_heatmaps/data_static_non_uniform/EGSC_ogbg-code2_heatmap_25.pdf}}
%\subfloat[\hmn]{\includegraphics[width =1.8in]{figures_heatmaps/data_static_non_uniform/H2MN_ogbg-code2_heatmap_25.pdf}}
%\subfloat[\graphedx]{\includegraphics[width =1.8in]{figures_heatmaps/data_static_non_uniform/GraphEDX_ogbg-code2_heatmap_25.pdf}}\\
\subfloat[\name]{\includegraphics[width =1.4in]{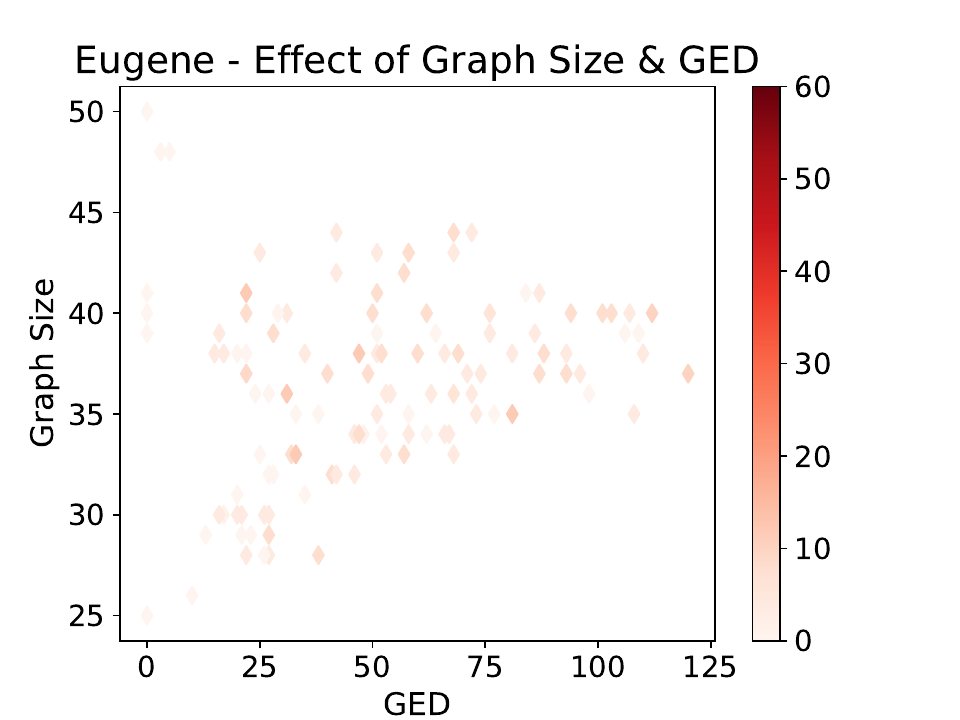}}
\subfloat[\egsc]{\includegraphics[width=1.4in]{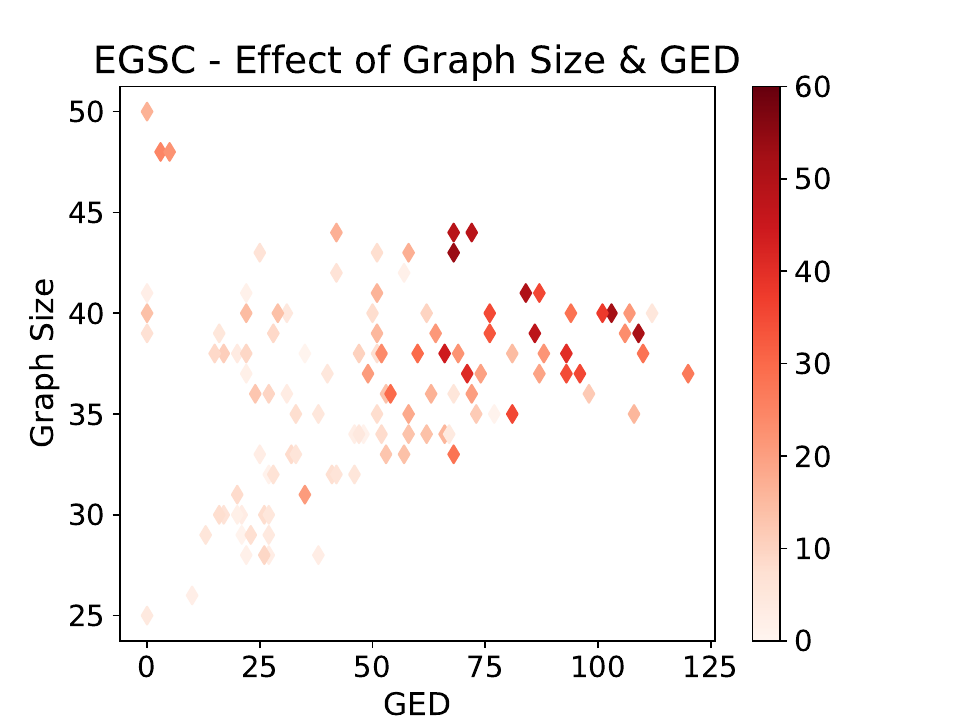}}
\subfloat[\hmn]{\includegraphics[width =1.4in]{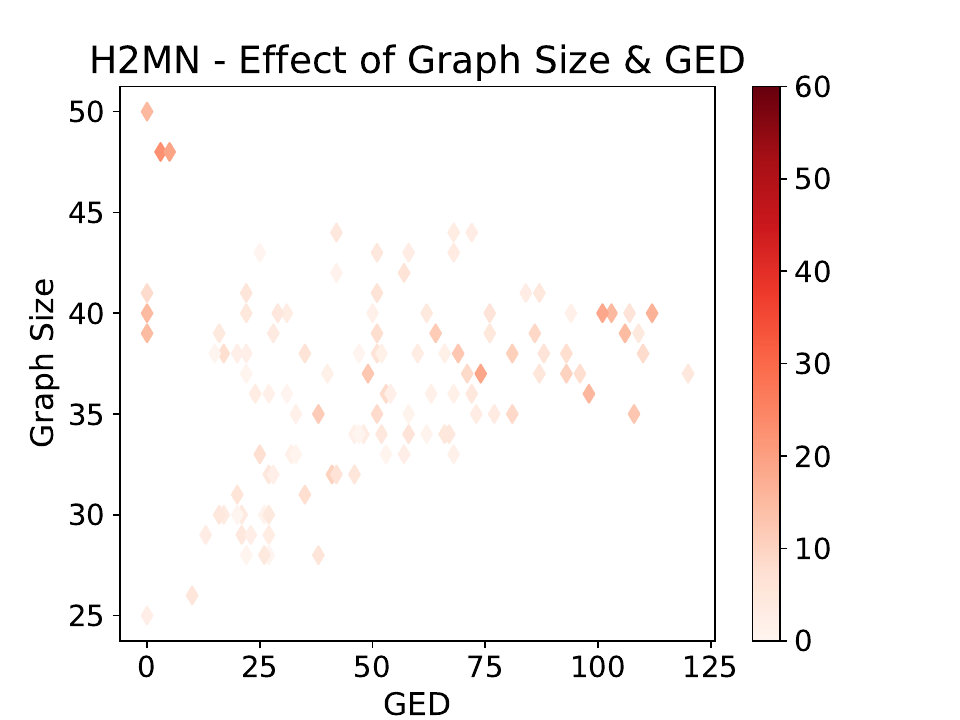}}
\subfloat[\graphedx]{\includegraphics[width =1.4in]{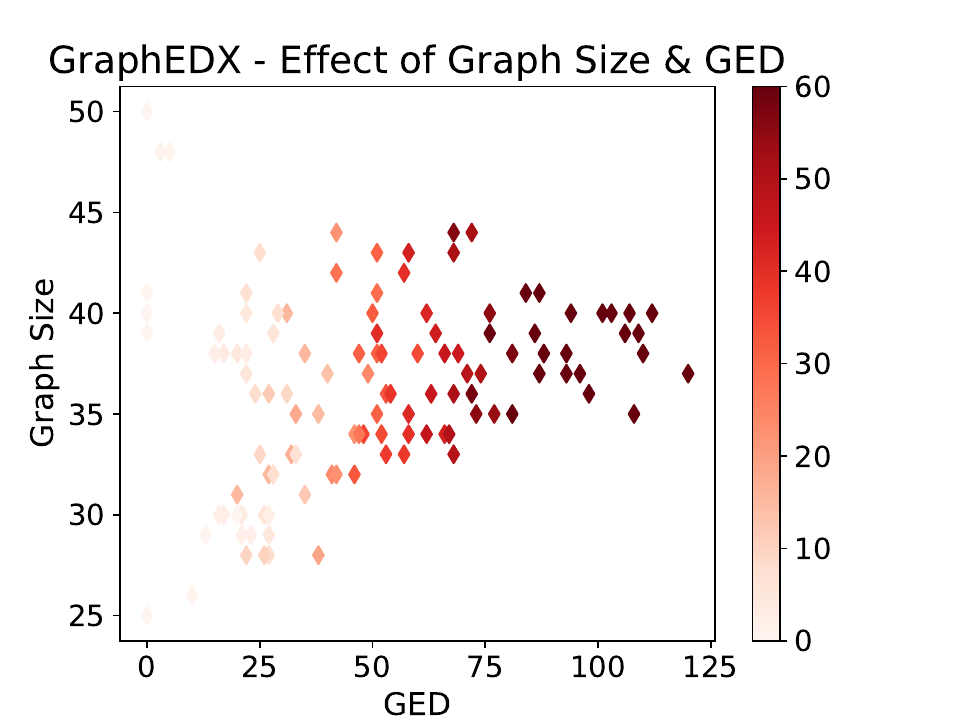}}
\vspace{-3mm}
\caption{MAE heatmap vs. graph size \& GED for \code for graphs of size $[25, 50]$.}\label{fig:heatmap_code2}
\vspace{-4mm}
\end{figure*}
\vspace{-0.1in}
\subsection{Comparison with \fugal}
\vspace{-0.1in}
\fugal addresses unrestricted graph alignment (UGA), while \name estimates GED and produces an alignment corresponding to the approximation. As Theorem~\ref{thm:equivalence} shows, UGA is a special case of GED with node edit costs set to zero. The connection between UGA and GED established in Theorem~\ref{thm:equivalence2} allows us to draw from UGA methods, though our optimization differs in key ways:

\begin{wraptable}{r}{0.55\textwidth}
\centering
\vspace{-0.2in}
\captionsetup{font=small}
\caption{GED estimation error (MAE) under Cost Setting~1.}
\label{tab:nodecost}
\scalebox{0.8}{
\begin{tabular}{lcccc}
\toprule
Method & \aids & \molhiv & \code & \mutag \\
\midrule
\fugal                 & 7.12 & 11.72 & 6.53 & 11.60 \\
\fugal-Node Edit Costs & 7.71 & 12.65 & 7.93 & 12.49 \\
\name                & \textbf{0.33} & \textbf{0.65} & \textbf{0.75} & \textbf{0.68} \\
\bottomrule
\end{tabular}}
\vspace{-0.1in}
\end{wraptable}
\vspace{-0.1in}
\paragraph{Optimization.} \name employs a modified Adam optimizer with a penalty method to enforce doubly stochastic constraints, whereas UGA methods typically use Frank-Wolfe~\cite{frank1956algorithm} with Sinkhorn-Knopp normalization~\cite{cuturi2013sinkhorn}. As shown in Table~\ref{tab:fw}, replacing Adam with Frank--Wolfe (\name-FW) leads to weaker performance, confirming the effectiveness of our approach. Our novel \emph{inverse relabelling} strategy further improves GED estimation (\S~\ref{app:ablation}).
\vspace{-0.1in}
\paragraph{Cost Regularizer.} \name integrates node edit costs through a matrix~$D$, while UGA methods may only use similar terms as structural regularizers. To test whether \fugal could benefit from node edit costs, we evaluated it with \name's cost matrix~$D$. Table~\ref{tab:nodecost} shows that both \fugal variants yield substantially higher GED error than \name.  

\begin{wraptable}{r}{0.5\textwidth}
\centering
\vspace{-0.25in}
\captionsetup{font=small}
\caption{GED estimation error (MAE) under zero node edit costs (UGA setting).}
\label{tab:zeronode}
\scalebox{0.8}{
\begin{tabular}{lcccc}
\toprule
Method      & \aids & \molhiv & \code & \mutag \\
\midrule
\fugal  & 4.71 & 6.98   & 8.52  & 8.44  \\
\name & \textbf{0.28} & \textbf{0.50} & \textbf{0.74} & \textbf{0.55} \\
\bottomrule
\end{tabular}}
\vspace{-0.1in}
\end{wraptable}

One might still believe that \fugal is inherently tailored for GED instances with zero node edit costs, corresponding to UGA. We thus set all node edit costs to~0 and edge edit costs to~1. Even under this UGA-compatible setting, \name demonstrated superior performance, as shown in Table~\ref{tab:zeronode}.

\begin{wraptable}{r}{0.5\textwidth}
\vspace{-0.2in}
\centering
\captionsetup{font=small}
\caption{Replacing \name’s Frobenius norm with FUGAL’s non-convex correlation term.}
\label{tab:core}
\vspace{-0.1in}
\scalebox{0.7}{
\begin{tabular}{lcccc}
\toprule
Method             & \aids & \molhiv & \code & \mutag \\
\midrule
\name (\fugal QAP) & 5.43 & 7.53   & 17.82 & 12.65 \\
\name             & \textbf{0.33} & \textbf{0.65} & \textbf{0.75} & \textbf{0.68} \\
\bottomrule
\end{tabular}}
\vspace{-0.1in}
\end{wraptable}

This raises the question of why the poor GED estimates from UGA methods are not evident in UGA studies. The key difference lies in evaluation: GED is evaluated strictly by edge and node differences from the ground truth (the QAP objective), while UGA is evaluated more loosely by the fraction of correctly aligned nodes. Hence, GED methods must enforce much stricter fidelity to the QAP objective than UGA methods, as we discuss in the following.

\paragraph{Core Objective Term.} \name prioritizes the convex Frobenius norm $\|A P - P B \|_F^2$,
which ensures stable updates. UGA methods instead optimize the non-convex correlation term $\mathrm{Tr}(A P B^\top P^\top)$ for efficiency, paired with Frank-Wolfe. Substituting this non-convex term into \name caused divergence; even the best result within a 10-minute cap (Table~\ref{tab:core}) remained far less accurate. This confirms that FUGAL’s core objective is ill-suited for GED estimation.
\vspace{-0.1in}
\section{Conclusions}\label{sec:conclusion}
\vspace{-0.1in}
\pke{We introduced \name, an optimization based heuristic method that provides explainable estimates of GED based on a structure-aware representation and relaxation of the underlying optimization problem. Through extensive experimentation, we demonstrated that \name achieves state-of-the-art GED estimates and superior scalability compared to baselines across diverse datasets, even while it eliminates the need to generate supervisory data via \NPhard computations. These features position \name as a promising candidate for practical graph similarity measurement. As our implementation relies solely on CPU resources, it is open to further enhancement.}
%Future work aims to extend \name to graph alignment tasks beyond GED estimation.

\section{Reproducibility Statement}
We have made the implementation of \name publicly available; the code link is provided at the end of Page~6. The released implementation includes the benchmark test sets, as well as the training and validation sets used for the neural models. We also provide scripts to generate new test sets for independent evaluation. Details on data generation, testing setup, and baseline implementations are described in Section~\ref{sec:exp}. Appendix~\ref{app:setup} specifies the hardware and software environment, and Appendix~\ref{app:parameters} lists the parameters used by \name.

% \clearpage
% \section{Other}

% \subsubsection{Appendices.}
% Any appendices must appear after the main content. If your main sections are numbered, appendix sections must use letters instead of arabic numerals. In \LaTeX{} you can use the \texttt{\textbackslash appendix} command to achieve this effect and then use \texttt{\textbackslash section\{Heading\}} normally for your appendix sections.

% \subsubsection{Ethical Statement.}
% You can write a statement about the potential ethical impact of your work, including its broad societal implications, both positive and negative. If included, such statement must be written in an unnumbered section titled \emph{Ethical Statement}.

\bibliographystyle{iclr26/iclr2026_conference}
\bibliography{references}
% clean the garbage!
%\bibliographystyle{icml/icml2025}
%\input{sections/99.references}
\clearpage
%\appendix
\onecolumn
\section{Appendix}\label{sec:appendix}
\renewcommand{\thesubsection}{\Alph{subsection}}
\renewcommand{\thefigure}{\Alph{figure}}
\renewcommand{\thetable}{\Alph{table}}
\begin{figure}[t]
\vspace{-2mm}
\centering
\includegraphics[width=3.3in]{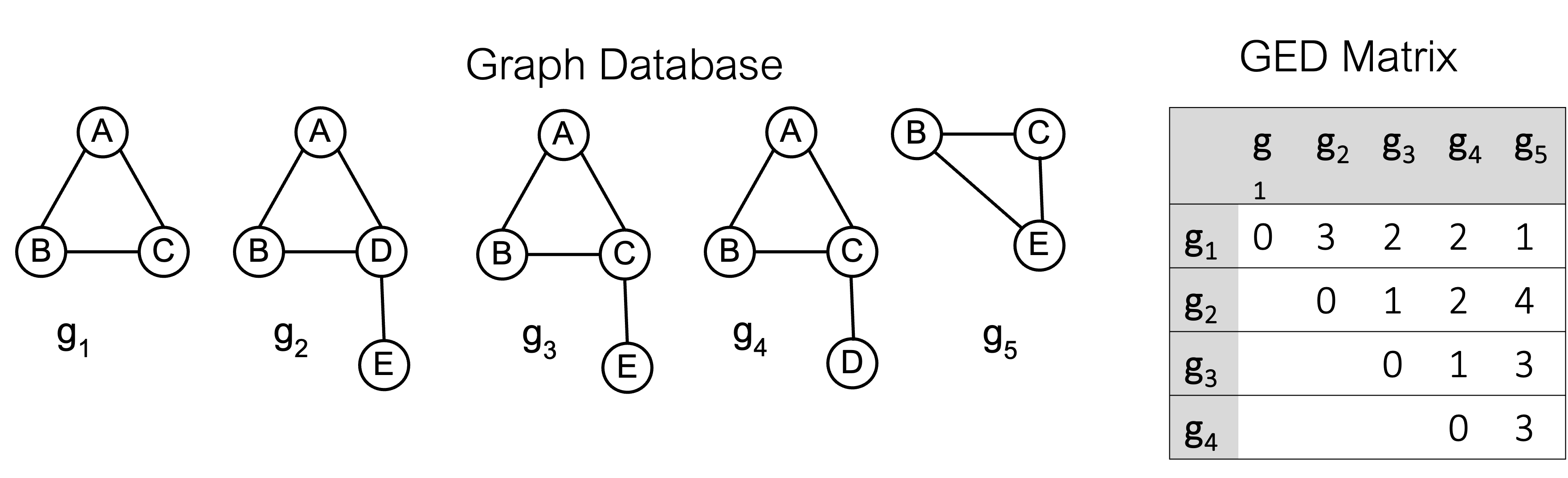}
\vspace{-2mm}
\caption{GED among five graphs; all edit operations cost~1.}\label{fig:ged}
\vspace{-2mm}
\end{figure}

\subsection{Related Work}

\textbf{Supervised Methods:} \graphedx~\cite{graphedx} represents each graph as a set of node and edge embeddings and learn the alignments using a Gumbel-Sinkhorn permutation generator, additionally ensuring that the node and edge alignments are consistent with each other. \greed~\cite{ranjan&al22} employs a siamese network to generate graph embeddings in parallel and estimates the Graph Edit Distance (GED) by computing the norm of their difference. \eric~\cite{eric} eliminates the need for explicit node alignment by leveraging a regularizer and computes similarity using a Neural Tensor Network (NTN) and a Multi-Layer Perceptron (MLP) applied to graph-level embeddings obtained from a Graph Isomorphism Network (GIN). \textsc{GMN}~\cite{icmlged} assess graph similarity using Euclidean distance between embeddings and exist in two variants: \gmn (late interaction) and \textsc{GMN-Match} (early interaction), both utilizing message passing to capture structural similarities.  \simgnn~\cite{simgnn} combines graph-level and node-level embeddings, where a Neural Tensor Network processes graph-level embeddings, while a histogram-based feature vector derived from node similarities enhances the similarity computation. \hmn~\cite{h2mn} utilizes hypergraphs to model higher-order node similarity, employing a subgraph matching module at each convolution step before aggregating the final graph embeddings via a readout function and passing them through an MLP. \egsc~\cite{egsc} introduces an Embedding Fusion Network (EFN) within a Graph Isomorphism Network (GIN) to generate unified embeddings for graph pairs, which are further processed through an EFN and an MLP to compute the final similarity score. \nips{\gotsim~\cite{graphotsim} approximates GED through a neural network, and simultaneously learns the alignments. Specifically, it formulates the similarity between a pair of graphs as the minimal “transformation” cost from one graph to another in the learnable node-embedding space. \gedgnn~\cite{gedgnn} treats GED computation as a regression task and predict the GED value. A post-processing algorithm based on $k$-best matching is used to extract node mapping. \gmsm~\cite{gmsm} uses regularized optimal transport with GNNs to approximate GED.}

\textbf{Heuristic Methods:} \lpged~\cite{LEROUGE2017254}, \adjip~\cite{adjip}, and \compact~\cite{blumenthal2020exact} employ a mixed integer programming framework based on the LP-GED paradigm to approximate the GED. In contrast, \btight~\cite{btight} iteratively solves instances of the linear sum assignment problem or the minimum-cost perfect bipartite matching problem. \nips{\ipfp~\cite{ipfp} models GED as a quadratic assignment problem and uses Integer Projected Fixed Point method to aproximate the QAP.}

\subsection{Proofs}\label{app:proofs}

\uga*
%\begin{proof}
%By Lemma~\ref{lemma:ged}, ED($\pi$) = $\|AP - PB\|^{2}$. It is GED = $\min_{\pi}$ ED($\pi$) = $\min_{\pi \in \phi(\CG_1,\CG_2)} \|AP_{\pi} - P_{\pi}B\|^{2}_{F}$, since there is a one-to-one mapping between alignment functions~$\pi$ and permutation matrices~$\pi$%.; then GED($\mathcal{G}_1$, $\mathcal{G}_2$) = CD($\mathcal{G}_1$, $\mathcal{G}_2$)$^{2}$.}
%\end{proof}
%We show that by expressing the objective of Graph Alignment between two graphs~$\mathcal{G}_{1}$ and~$\mathcal{G}_{2}$, defined in Equation~\eqref{cd}, as an instance of GED. Consider an instance of GED between graphs~$\mathcal{G}_{1}$ and~$\mathcal{G}_{2}$ whose adjacency matrices are~$A, B$ respectively, 
% defined as follows:}
% \begin{itemize}
%     \item Node insertion cost = 0
%     \item Node deletion cost = 0
%     \item Node substitution cost = 0
%     \item Edge insertion cost = 2
%     \item Edge deletion cost = 2
%     \item Edge substitution cost = 0
% \end{itemize}
%\nips{We show that the GED between~$\mathcal{G}_{1}$ and~$\mathcal{G}_{2}$ by these edit costs is equal to the objective of Graph Alignment.}
%\begin{lemma}\label{lemma:ged}
%\pk{Given graphs~$\mathcal{G}_{1}$, $\mathcal{G}_{2}$ of size~$n$ and a node alignment function~$p: [n] \rightarrow [n]$ mapping each node~$i$ in~$\mathcal{G}_{1}$ to a node~$\pi(i)$ in~$\mathcal{G}_{2}$, the edit distance corresponding to the node alignment by the aforementioned edit costs is equal to~$\|AP - PB\|^{2}_{F}$, where~$\pi$ is a permutation matrix with~$P_{ij} = 1$ if~$\pi(i) = j$, otherwise~0.}
%\end{lemma}
\vspace{-3mm}
\begin{proof}
% In graph alignment we want to find the $P$ (equivalently $\pi$) minimizing Eq.~\ref{eq:alignment}.
\pke{We derive the set of edge insertions and deletions to convert~$\mathcal{G}_{1}$ to~$\mathcal{G}_{2}$ from~$\pi$. An edge that should be inserted between nodes~$i$ and~$j$ in~$\mathcal{G}_{1}$ does not exist in~$A$ but exists in~$B$, hence~$a_{ij} = 0$ and~$b_{\pi(i)\pi(j)} = 1$. Likewise, an edge that needs deletion has~$a_{ij} = 1$ and~$b_{\pi(i)\pi(j)} = 0$. All other~$(i, j)$ pairs have~$a_{ij} = b_{\pi(i)\pi(j)}$. Let~$\mathcal{E}_{ins}$ be the set of edges to be inserted in~$\mathcal{G}_{1}$ and~$\mathcal{E}_{del}$ that of edges to be deleted by~$\pi$, where without loss of generality an edge~$(i,j)$ has~$i < j$. As node edit and edge substitutions cost~$0$, the~$GED_\pi(\mathcal{G}_{1},\mathcal{G}_{2})$ with respect to edit operations induced by~$\pi$ is:}
%\[
%\text{ED($\pi$)} = \sum_{\mathclap{(i, j) \in \mathcal{E}_{ins}}} c(\edgeins(\pi(i), \pi(j))) + \sum_{\mathclap{(i, j) \in \mathcal{E}_{del}}} c(\edgedel(i, j))
%\]
%In this expression, $c(\edgeins(\pi(i), \pi(j)))$ and~$c(\edgedel(i, j))$ denote the costs of edge insertion and deletion between nodes~$i$ and~$j$, respectively. By the defined costs, we get:
{\small
\begin{multline}
\nonumber
GED_\pi(\mathcal{G}_1, \mathcal{G}_2) = \sum_{(i, j) \in \mathcal{E}_{ins}} 2 + \sum_{(i, j) \in \mathcal{E}_{del}} 2
= \sum_{(i, j) \in \mathcal{E}_{ins}} 2 \cdot (a_{ij} - b_{\pi(i)\pi(j)})^2 
+ \sum_{(i, j) \in \mathcal{E}_{del}} 2 \cdot (a_{ij} - b_{\pi(i)\pi(j)})^2 \\
= \sum_{(i, j) \in \mathcal{E}_{ins}} \big((a_{ij} - b_{\pi(i)\pi(j)})^2 + (a_{ji} - b_{\pi(j)\pi(i)})^2\big) \hfill
% \end{multline}
% \begin{multline}
+ \sum_{(i, j) \in \mathcal{E}_{del}} \big((a_{ij} - b_{\pi(i)\pi(j)})^2 + (a_{ji} - b_{\pi(j)\pi(i)})^2\big) \hfill \\
+ \sum_{i < j, (i, j) \notin \mathcal{E}_{del} \cup \mathcal{E}_{ins}} \big((a_{ij} - b_{\pi(i)\pi(j)})^2 + (a_{ji} - b_{\pi(j)\pi(i)})^2\big) \hfill\\
= \sum_{(i, j) \in [n] \times [n]} (a_{ij} - b_{\pi(i)\pi(j)})^2
= \|A - P_\pi BP_\pi^T\|_F^2 = \|AP_\pi - P_\pi B\|_F^2 \hfill
\label{eq:ptoP}
\end{multline}}

\pke{By the given edit costs, $GED(\mathcal{G}_{1},\mathcal{G}_{2}) = \min_{\pi} \left\{GED_\pi(\mathcal{G}_{1},\mathcal{G}_{2})\right\}$, hence,}

\begin{equation*}
GED(\mathcal{G}_{1},\mathcal{G}_{2}) = \min_{\pi \in \Phi(\CG_1,\CG_2)} \|AP_{\pi} - P_{\pi}B\|^{2}_{F}
\end{equation*}
\end{proof}

\geduga*
\begin{proof}
%Similar to the proof of Thm.~\ref{thm:equivalence}, let $\pi: \CV_1 \rightarrow \CV_2$ be a node alignment function, where~$\pi(i) = j$ if and only if~$P_{ij} = 1$.
\pke{We first reformulate Equation~\ref{ged_chdist} as follows:}
\[
\scalebox{0.99}{$\frac{||\tilde{A}P_{\pi} - P_{\pi}\tilde{B}||_F^{2}}{2} + tr(P_{\pi}^{T}D) = \frac{||\tilde{A} - P_{\pi}\tilde{B}P_{\pi}^{T}||_F^{2}}{2} + tr(P_{\pi}^{T}D)$}
\]
\pke{Using the node-alignment function~$\pi$, we reformulate the above to:}
\[
\textstyle\sum_{(i, j) \in [n] \times [n]} \kappa^2 \cdot \frac{(a_{ij} - b_{\pi(i)\pi(j)})^2}{2} + \sum_{i \in [n]} d_{i, \pi(i)}
\]
\pke{Further manipulation via the definition of matrix~$D$ gives:}
\[
\begin{split}
\small \sum_{\substack{(i, j) \in [n] \times [n]}} \!\!\!\!\!\!\! \kappa^2 \cdot \frac{(a_{ij} - b_{\pi(i)\pi(j)})^2}{2} +\!\!\!\!\!\!\!\!
\sum_{\substack{i \in \mathcal{G}_{1} \footnotesize \text{is a dummy}}} \!\!\!\!\! d_{v}(\epsilon, \mathcal{L}(\pi(i))) +\!\!\!\!\!\!\!\!\!\!\!\!\!\!\!\!
\sum_{\substack{i \in \mathcal{G}_{1} \footnotesize \text{mapped to dummy } \pi(i)}} \!\!\!\!\!\!\!\!\!\!\!\!\!\!\! d_{v}(\mathcal{L}(i), \epsilon) +\!\!\!\!\!\!\!\!
\sum_{\substack{\mathcal{L}(i) \neq \mathcal{L}(\pi(i))}} \!\!\!\!\!\! d_{v}(\mathcal{L}(i), \mathcal{L}(\pi(i)))
\vspace{-3mm}
\end{split}
\]
\pke{Notably, for any~$(i, j) \in [n] \!\times\! [n]$, if~$a_{ij} \!=\! 0$ and~$b_{\pi(i)\pi(j)} \!=\! 1$, an~$(i, j)$ edge should be inserted. Likewise, if~$a_{ij} \!=\! 1$ and~$b_{\pi(i)\pi(j)} \!=\! 0$, edge~$(i, j)$ should be deleted. Otherwise, if~$a_{ij} \!=\! b_{\pi(i)\pi(j)}$, the term evaluates to~0. Besides, a dummy node~$i$ in~$\mathcal{G}_{1}$ should be inserted with~$\pi(i)$ as the corresponding node in~$\mathcal{G}_{2}$, while a node~$i$ mapped to a dummy node~$\pi(i)$ should be deleted. In the event that none of these conditions apply, node~$i$ is substituted with node~$\pi(i)$. We thus simplify the expression to:}
\vspace{-1mm}
\[
\begin{split}
\small
\sum_{\mathclap{(i, j) \text{ inserted}}} \kappa^2 \cdot \frac{b_{\pi(i)\pi(j)}^2 + b_{\pi(j)\pi(i)}^2}{2} + \sum_{\mathclap{(i, j) \text{ deleted}}} \kappa^2 \cdot \frac{a_{ij}^2 + a_{ji}^2}{2} + 
\sum_{\substack{i \in \mathcal{G}_{1} \text{is inserted}}} d_{v}(\epsilon, \mathcal{L}(\pi(i)))
+\\ \sum_{\substack{i \in \mathcal{G}_{1} \text{is deleted}}} d_{v}(\mathcal{L}(i), \epsilon) + \sum_{\substack{i \in \mathcal{G}_{1}\text{ is replaced} \text{with }\pi(i)}}d_{v}(\mathcal{L}(i), \mathcal{L}(\pi(i)))
\end{split}
\]
\pke{Substituting the values, we obtain:}
% \vspace{-0.1in}
\begin{alignat}{2}
\nonumber
\begin{split}
\sum_{\substack{(i, j) \text{ inserted}}} \!\!\!\!\!\! \kappa^2 + \!\!\!
\sum_{\substack{(i, j) \text{ deleted}}}  \!\!\!\!\!\! \kappa^2 + \!\!\!
%\sum_{(i, j) \text{ is substituted}} 0 +
\sum_{\substack{i \in \mathcal{G}_{1} \text{is inserted}}} \!\!\!\!\!\!\!\! d_{v}(\epsilon, \mathcal{L}(\pi(i))) + \!\!\!
\sum_{\substack{i \in \mathcal{G}_{1} \text{is deleted}}} \!\!\!\!\!\!\!\! d_{v}(\mathcal{L}(i), \epsilon) + \!\!\!\!\!\!\!\!\!\!\!\!\!\!\!
\sum_{\substack{i \in \mathcal{G}_{1}\text{ is replaced with }\pi(i)}} \!\!\!\!\!\!\!\!\!\!\!\!\!\! d_{v}(\mathcal{L}(i), \mathcal{L}(\pi(i)))
\end{split}\\
&=GED_{\pi}(\mathcal{G}_1,\mathcal{G}_2)
\end{alignat}
Since \small{$GED(\mathcal{G}_{1},\mathcal{G}_{2}) = \min_{\pi} \left\{GED_\pi(\mathcal{G}_{1},\mathcal{G}_{2})\right\}$} and \small{$\min_{\substack{\pi \in \Phi(\CG_1,\CG_2)}} \frac{||\tilde{A}P_{\pi} - P_{\pi}\tilde{B}||_F^{2}}{2} + tr(P_{\pi}^{T}D)= GED_\pi(\mathcal{G}_{1},\mathcal{G}_{2})$, $GED(\mathcal{G}_{1},\mathcal{G}_{2}) =\min_{\pi \in \Phi(\CG_1,\CG_2)} \frac{||\tilde{A}P_{\pi} - P_{\pi}\tilde{B}||_F^{2}}{2} + tr\left(P_{\pi}^{T}D\right)$. }
%Furthermore, from Eq.~\ref{eq:ptoP}, $\min_{\pi}\left\{GED_\pi(\mathcal{G}_{1},\mathcal{G}_{2})\right\}=\min_{P \in \mathbb{P}^{n}}\|AP - PB\|^{2}$. Hence, $GED(\mathcal{G}_{1},\mathcal{G}_{2}) = \min_{P \in \mathbb{P}^{n}}\|AP - PB\|^{2}$.
%This result is the \textsc{Mc-ED} matching node alignment~$p$.
\end{proof}
\doublystoch*
\begin{proof}
Given that~$tr(A^{T}(\pk{J} \!-\! A)) \!=\! 0$, it follows that $\sum_{i} \sum_{j} a_{ij} \cdot (1 \!-\! a_{ij}) \!=\! 0$. Since~$A$ is doubly-stochastic, $0 \leq a_{ij} \leq 1$ for all~$i$ and~$j$, hence~$a_{ij} \!\cdot\! (1 \!-\! a_{ij})$ is non-negative for~$1 \leq i,j \leq n$. Thus, $a_{ij} \!\cdot\! (1 \!-\! a_{ij}) \!=\! 0$ for all~$i$ and~$j$. It follows that~$a_{ij}$ must be either 0 or 1 for each~$i$ and~$j$. As~$A$ is doubly-stochastic and all its entries are either 0 or 1, by definition~$A$ is a permutation matrix.
\end{proof}

\convexity*
\begin{proof}
We begin by considering the first term in Equation~\eqref{relaxation2_v2}, \(\frac{1}{2} \| \tilde{A}P - P\tilde{B} \|^2\). The second derivative of this term is given by: $I \otimes \tilde{A}^2 - 2 \cdot (\tilde{B} \otimes \tilde{A}) + \tilde{B}^2 \otimes I$, where \(\otimes\) denotes the Kronecker product, and \(I\) represents the identity matrix. The second term in the equation is linear in the matrix \(P\), implying that its second derivative is zero. The second derivative of the third term is given by: $-2\lambda (I \otimes I)$.
Thus, the Hessian matrix of the entire function is:
\[
I \otimes \tilde{A}^2 - 2 \cdot (\tilde{B} \otimes \tilde{A}) + \tilde{B}^2 \otimes I - 2\lambda (I \otimes I).
\]
For the function to be convex, the Hessian must be positive semidefinite, which requires that its eigenvalues be non-negative. This leads to the condition:
\begin{equation}
    \small{\lambda \leq \frac{\lambda_{i}(\tilde{A})^2 + \lambda_{j}(\tilde{B})^2 - 2\lambda_{i}(\tilde{A})\lambda_{j}(\tilde{B})}{2} = \frac{(\lambda_{i}(\tilde{A}) - \lambda_{j}(\tilde{B}))^2}{2}},
\end{equation}
for all \(i, j \in \{1, 2, \dots, n\}\), where \(\lambda_{i}(\tilde{A})\) and \(\lambda_{j}(\tilde{B})\) are the eigenvalues of matrices \(\tilde{A}\) and \(\tilde{B}\), respectively.
\end{proof}

\subsection{Experiments}
\subsubsection{Hardware and Software environments}
\label{app:setup}
We ran all experiments on a machine equipped with an Intel Xeon Gold 6142 CPU @1GHz and a GeForce GTX 1080 Ti GPU. While heuristic methods including \name run on the CPU, supervised baselines exploit the GPU. 
\subsubsection{Datasets}\label{app:datasets}
The semantics of the datasets are as follows:
\vspace{-0.05in}
\begin{itemize}
\vspace{-0.05in}
    \item \textbf{\aids}~\cite{aids}: A compilation of graphs originating from the AIDS antiviral screen database, representing chemical compound structures.
    % \vspace{-0.02in}
    \item \textbf{OGBG-\molhiv} (\molhiv)~\cite{ogbg}: Chemical compound datasets of various sizes, where each graph represents a molecule. Nodes correspond to atoms, and edges represent chemical bonds. The atomic number of each atom serves as the node label.
    % \vspace{-0.02in}
    \item \textbf{OGBG-\code} (\code)~\cite{ogbg}: A collection of Abstract Syntax Trees (ASTs) derived from approximately 450,000 Python method definitions. Each node in the AST is assigned a label from a set of 97 labels. We considered the graphs as undirected.
    % \vspace{-0.02in}
    \item \textbf{Mutagenicity} (\mutag)~\cite{mutag}: A chemical compound dataset of drugs categorized into two classes: mutagenic and non-mutagenic.
    % \vspace{-0.02in}
    \item \textbf{\imdb}~\cite{imdb}: This dataset consists of ego-networks of actors and actresses who have appeared together in films. The graphs in this dataset are unlabelled.
    % \vspace{-0.02in}
    \item \textbf{\coil}~\cite{coildel}: This dataset comprises graphs extracted from images of various objects using the Harris corner detection algorithm. The resulting graphs are unlabelled.
    % \vspace{-0.02in}
    \item \textbf{Triangles}~\cite{triangles}: This is a synthetically generated dataset designed for the task of counting triangles within graphs. The graphs in this dataset are unlabelled.
\end{itemize}

\subsubsection{Parameters}\label{app:parameters}
Table~\ref{tab:parameters} lists the parameters used for \name. We set the convergence criterion of \textsc{M-Adam} to~$abs(prev\_dist - cur\_dist) < 1e^{-7}$, where~$prev\_dist, cur\_dist$ are the approximated Graph edit distances in two successive iterations, $itr - 1$ and~$itr$.

\begin{table}[!h]
% \vspace{-2mm}
\caption{Parameters used in \name.}\label{tab:parameters}
% \vspace{-3mm}
    \centering
    \begin{tabular}{p{3.6cm}p{3.6cm}}
    \toprule
    parameter & value\\
    \midrule
    $\mu$ & 1 \\
    $\alpha$ & 0.001 \\
    $\sigma_{th}$ & $1e^3$ \\
    \bottomrule
    \end{tabular}
\end{table}
\vspace{-3mm}
\definecolor{1st}{rgb}{0.8, 1, 0.5}
\definecolor{2nd}{rgb}{1.0, 0.9, 0.3}
\definecolor{3rd}{rgb}{0.9,1,0.7}

\begin{table*}[t] % we need !h for compression purposes
\caption{Accuracy Comparison among baselines for unit edit costs. Cells shaded in green denote the best performance in each dataset.}\label{tab:accuracy_uniform}
 \vspace{-2mm}
% \caption{RMSE on large datasets.}\label{tab:accuracy2}
% \vspace{-2mm}
\centering
\scalebox{0.9}{
    \begin{tabular}{p{2.5cm}|p{1.2cm}p{1.2cm}p{1.2cm}p{1.2cm}|p{1.2cm}p{1.2cm}p{1.2cm}p{1.2cm}} 
      \toprule
      & \multicolumn{4}{c|}{MAE} & \multicolumn{4}{c}{SI}\\
      \textbf{Methods} & \aids & \molhiv & \code & \mutag & \aids & \molhiv & \code & \mutag\\
      \midrule
      \eric & 0.57 & 0.66 & \cellcolor{1st}{0.56} & 0.65 & 0.00 & 0.00 & 0.00 & 0.00\\
      \egsc & 0.70 & 0.81 & 0.80 & 0.82 & 0.00 & 0.00 & 0.00 & 0.00\\
      \graphedx & 0.65 & 0.85 & 0.59 & 0.78 & 0.00 & 0.00 & 0.00 & 0.00\\
      \hmn & 0.86 & 0.94 & 0.84 & 0.89 & 0.00 & 0.00 & 0.00 & 0.00\\
      \gmn & 0.61 & 0.75 & 0.76 & 1.15 & 0.00 & 0.00 & 0.00 & 0.00\\
      \greed & 0.59 & 0.82 & 0.75 & 0.75 & 0.00 & 0.00 & 0.00 & 0.00\\
      \simgnn & 0.77 & 0.90 & 0.79 & 1.06 & 0.00 & 0.00 & 0.00 & 0.00\\
      \midrule
      \gedgnn & 1.19 & 2.16 & 1.50 & 1.89 & 0.00 & 0.00 & 0.00 & 0.00\\
      \gotsim & 3.36 & 5.20 & 9.76 & 4.74 & 0.00 & 0.00 & 0.00 & 0.00\\
      \gmsm & 7.34 & 13.04 & 10.01 & 13.32 & 0.00 & 0.00 & 0.00 & 0.00\\
      \midrule
      \btight & 4.13 & 4.98 & 6.79 & 7.05 & 0.02 & 0.02 & 0.02 & 0.01\\
      \adjip & 0.45 & 2.16 & 2.32 & 2.27 & 0.83 & 0.69 & 0.50 & 0.62\\
      \lpged & 2.6 & 5.48 & 2.82 & 5.39 & 0.48 & 0.13 & 0.14 & 0.05\\
      \compact & 1.49 & 4.17 & 3.93 & 4.07 & 0.75 & 0.27 & 0.01 & 0.18\\
      \ipfp & 2.81 & 5.19 & 2.85 & 4.97 & 0.08 & 0.02 & 0.14 & 0.02\\
      \midrule
      \name & \cellcolor{1st}{0.26} & \cellcolor{1st}{0.55} & 0.72 & \cellcolor{1st}{0.58} & \cellcolor{1st}{0.87} & \cellcolor{1st}{0.74} & \cellcolor{1st}{0.69} & \cellcolor{1st}{0.72}\\
      \bottomrule
    \end{tabular}}
    % \vspace{-0.1in}
\end{table*}

\subsubsection{Accuracy under Uniform Edit Cost Setting}\label{app:acc_uniform}  
Table~\ref{tab:accuracy_uniform} presents the approximation accuracy results in terms of MAE and SI on benchmark datasets under the uniform cost setting (Case 3). For MAE, \name outperforms all baselines on the \aids, \molhiv, and \mutag datasets, while on the \code dataset, \eric outperforms \name. In terms of SI, \name consistently surpasses all considered baselines. These results establish \name as a robust method capable of accurately estimating GED across diverse cost settings. \nips{The difficulty (i.e., MAE) increases as costs become more diverse (i.e., from uniform to non-uniform costs) and the size of the considered edit space expands (i.e., from zero to non-zero cost of substitution). We thus observe the lowest MAE in Setting 3, followed by Setting 1, and the highest MAE in Setting 2.}

\definecolor{1st}{rgb}{0.8, 1, 0.5}
\definecolor{2nd}{rgb}{1.0, 0.9, 0.3}
\definecolor{3rd}{rgb}{0.9,1,0.7}

\begin{table*}[t] % we need !h for compression purposes
\caption{Accuracy Comparison among baselines in terms of MAE under different cost settings for unlabelled datasets. Cells shaded in greendenote the best performance in each dataset.}\label{tab:accuracy_static_uniform_unlabelled}
 % \vspace{-2mm}
% \caption{RMSE on large datasets.}\label{tab:accuracy2}
% \vspace{-2mm}
\centering
\scalebox{0.9}{
    \begin{tabular}{p{2.6cm}|p{1.7cm}p{1.7cm}p{1.7cm}|p{1.7cm}p{1.7cm}p{1.7cm}} 
      \toprule
      & \multicolumn{3}{c|}{Cost Setting Case 1} & \multicolumn{3}{c}{Cost Setting Case 3}\\
      \textbf{Methods} & \imdb & \coil & \triangles & \imdb & \coil & \triangles\\
      \midrule
      \eric & 10.42 & 1.41 & 2.65 & 3.80 & 1.87 & 1.47\\
      \egsc & 5.96 & 3.23 & 3.80 & 6.50 & 3.89 & 2.82\\
      \graphedx & 7.10 & 1.41 & 2.26 & 1.46 & 1.21 & 0.50\\
      \hmn & 15.51 & 8.44 & 7.02 & 7.20 & 4.27 & 3.38\\
      \gmn & 4.75 & 2.93 & 3.41 & 1.37 & 0.89 & 0.63\\
      \greed & 5.02 & 2.90 & 3.39 & 1.39 & 0.88 & 0.73\\
      \simgnn & 7.58 & 2.00 & 2.36 & 3.73 & 1.04 & 0.97\\
      \midrule
      \gedgnn & 10.78 & 3.54 & 1.97 & 3.31 & 1.69 & 1.16\\
      \gotsim & 25.01 & 9.41 & 6.94 & 8.20 & 4.19 & 2.84\\
      \gmsm & 40.70 & 20.18 & 16.94 & 19.67 & 9.97 & 8.20\\
      \midrule
      \btight & 7.22 & 6.47 & 5.68 & 3.58 & 3.30 & 2.71\\
      \adjip & 1.58 & 0.71 & 0.40 & 1.22 & 0.23 & 0.30\\
      \lpged & 8.68 & 3.75 & 1.58 & 4.26 & 1.75 & 0.82\\
      \compact & 17.05 & 4.01 & 1.04 & 9.56 & 2.10 & 0.64\\
      \ipfp & 18.87 & 8.67 & 7.04 & 9.15 & 4.27 & 3.47\\
      \midrule
      \name & \cellcolor{1st}{1.02} & \cellcolor{1st}{0.43} & \cellcolor{1st}{0.21} & \cellcolor{1st}{0.15} & \cellcolor{1st}{0.21} & \cellcolor{1st}{0.17}\\
      \bottomrule
    \end{tabular}}
    \vspace{-0.1in}
\end{table*}

\subsubsection{Accuracy on Unlabelled Datasets}\label{app:unlabelled}
Table~\ref{tab:accuracy_static_uniform_unlabelled} presents the accuracy comparison of \imdb, \coil, and \triangles datasets in terms of MAE for cost setting Case 1 and Case 3. As these datasets are unlabelled, Case 2 is not applicable. \name consistently outperforms both supervised and heuristic baselines across all scenarios, demonstrating its robustness and effectiveness for GED prediction across diverse datasets.
\definecolor{1st}{rgb}{0.8, 1, 0.5}
\definecolor{2nd}{rgb}{1.0, 0.9, 0.3}
\definecolor{3rd}{rgb}{0.9,1,0.7}

\begin{table*}[t] % we need !h for compression purposes
\caption{Accuracy comparison among baselines in terms of SI under different cost settings for graphs of sizes~$[25,50]$. Cells shaded in green  denote the best performance in each dataset.}\label{tab:si_unseen}
 \vspace{-2mm}
% \caption{RMSE on large datasets.}\label{tab:accuracy2}
% \vspace{-2mm}
\centering
\scalebox{0.9}{
    \begin{tabular}{p{2.5cm}|p{1.2cm}p{1.2cm}p{1.2cm}p{1.2cm}|p{1.2cm}p{1.2cm}p{1.2cm}p{1.2cm}} 
      \toprule
      & \multicolumn{4}{c|}{Cost Setting Case 1} & \multicolumn{4}{c}{Cost Setting Case 2}\\
      \textbf{Methods} & \aids & \molhiv & \code & \mutag & \aids & \molhiv & \code & \mutag\\
      \midrule
      \btight & 0.12 & 0.03 & 0.05 & 0.09 & 0.01 & 0.05 & 0.04 & 0.04  \\
      \adjip & 0.18 & 0.03 & 0.09 & 0.10 & 0.25 & 0.08 & 0.13 & 0.10  \\
      \lpged & 0.04 & 0.00 & 0.01 & 0.04 & 0.03 & 0.04 & 0.10 & 0.03 \\
      \compact & 0.00 & 0.00 & 0.00 & 0.01 & 0.05 & 0.05 & 0.05 & 0.04  \\
      \ipfp & 0.01 & 0.01 & 0.00 & 0.01 & 0.00 & 0.00 & 0.01 & 0.00\\
      \midrule
      \name & \cellcolor{1st}{0.35} & \cellcolor{1st}{0.31} & \cellcolor{1st}{0.30} & \cellcolor{1st}{0.46} & \cellcolor{1st}{0.36} & \cellcolor{1st}{0.26} & \cellcolor{1st}{0.18} & \cellcolor{1st}{0.16} \\
      \bottomrule
    \end{tabular}}
    \vspace{-0.1in}
\end{table*}

\subsubsection{SI on Large Graphs}\label{app:si_unseen}  
Table~\ref{tab:si} presents a comparison of \name with other baselines in terms of the Strict Interpretability (SI) metric for graphs of sizes $[25, 50]$. \name consistently achieves significantly higher SI scores compared to other heuristic methods. These superior SI scores on large graphs highlight \name's enhanced scalability in delivering interpretable GED, outperforming other non-neural methods. 

\subsubsection{Carbon Emissions}
\vspace{-0.05in}
\begin{table}[b]
    \centering
    % \vspace{-0.2in}
    \caption{Total Carbon Emissions (in grams of $CO_{2}$).}
    % \vspace{-2mm}
    \scalebox{0.9}{
    \begin{tabular}{ccccc}
        \toprule
        \textbf{Model} & \aids & \molhiv & \code & \mutag \\
        \midrule
        \eric & 75.56 & 204.22 & 71.42 & 222.82 \\
        \egsc & 78.43 & 215.27 & 73.23 & 223.85 \\
        \graphedx & 410.65 & 612.09 & 426.40 & 251.55 \\
        \hmn & 437.91 & 442.55 & 123.21 & 277.00 \\
        \name & 6.06 & 7.11 & 8.11 &  7.28\\
        \bottomrule
    \end{tabular}}
    \label{tab:total_carbon_emissions}
\end{table}
Table~\ref{tab:total_carbon_emissions} presents the total carbon emissions for the top-performing models across various datasets. \name was executed on a CPU, which operates at a power consumption of approximately 150 watts under full load. In contrast, all other neural models utilized a GPU, which consumes approximately 250 watts under full load. \nips{Our carbon emission estimation follows a standard methodology:  
\begin{align*}  
\text{Energy Consumption} &= \text{Power (kW)} \times \text{Time (hours)} \\  
\text{CO}_2 \text{ Emissions} &= \text{Energy Consumption} \times 475 \text{ gCO}_2\text{/kWh}  
\end{align*}  
The emission factor of 475 gCO$_2$/kWh is sourced from~\citet{iea2019status}}. The carbon emissions account for the time taken to generate ground truth, training, and inference for the neural models, whereas \name, being optimization-based, only includes inference time. \nips{While we acknowledge that training and ground-truth computation costs would be amortized over many inferences, it is reasonable to include those costs for any model that requires them.} \name demonstrates significantly lower carbon emissions compared to the supervised methods, achieving up to 30 times lower emissions on the \molhiv dataset.

\begin{table}[t]
\caption{Running times (MM:SS) on benchmark datasets.}\label{tab:runtimes_unsup}
% \caption{\texttt{ogbg-molhiv}}\label{tab:molhiv_runtimes}
% \vspace{-2mm}
\centering
\scalebox{0.95}{
    \begin{tabular}{p{2.5cm}p{1.2cm}p{1.3cm}p{1.2cm}p{1.3cm}p{1.3cm}p{1.7cm}p{1.3cm}}
        \toprule
        \textbf{Methods} & \aids & \molhiv & \code & \mutag & \imdb & \coil & \triangles\\
        \midrule
        \name & 05:06 & 05:59 & 06:50 & 06:05 & 05:09 & 05:06 & 04:59\\
        \btight & 00:24 & 00:48 & 03:52 & 01:17 & 00:18 & 00:07 & 00:13\\
        \adjip & 02:38 & 06:34 & 09:33 & 07:12 & 05:42 & 02:06 & 01:29\\
        \ipfp & 00:20 & 00:45 & 01:40 & 01:20 & 00:15 & 00:05 & 00:09\\
        \compact & 10:02 & 11:47 & 12:53 & 12:16 & 07:35 & 07:38 & 05:01\\
        \lpged & 08:04 & 11:01 & 11:52 & 10:51 & 09:56 & 08:13 & 05:11\\
        \bottomrule
    \end{tabular}
}
\end{table}
\subsubsection{Efficiency}  
Table~\ref{tab:runtimes_unsup} presents the running time of optimization based heuristic methods on the benchmark datasets for the entire test set. Among these methods, \ipfp and \btight demonstrates the fastest runtimes but exhibits the poorest accuracy among all 15 baselines in Table~\ref{tab:accuracy_static_non_uniform} across datasets and cost settings. Excluding \btight and \ipfp, \name achieves superior runtime performance compared to other optimization based methods on the \molhiv, \code, \mutag and \imdb datasets. On the \aids, \coil and \triangles datasets, \adjip demonstrates better run times, and \name is second best. Importantly, \name achieves a significant accuracy advantage while maintaining competitive efficiency, reinforcing its position as both an effective and efficient solution for GED approximation.

\textbf{Time Complexity Analysis:} The objective function (Eq.~\eqref{relaxation2_v2}) includes matrix multiplications with a worst-case time complexity of $\mathcal{O}(n^3)$. Gradient calculations also have a worst-case complexity of $\mathcal{O}(n^3)$ due to matrix multiplications. Thus, the overall time complexity becomes $\mathcal{O}(T \cdot n^3)$, where $T$ is the number of computation epochs. Additionally, as the algorithm is CPU-bound, GED computations for each graph pair can be massively parallelized by leveraging multi-core CPUs and hyperthreading.

\iclr{\textbf{Impact of Time Budgets:} As certain heuristic baselines employ time constraints, we retained their default parameter settings to ensure consistency. To examine how performance varies with increased computational budget, we conducted an analysis on the \code dataset under \textbf{Cost Setting~1}, using time budgets of 5, 10, and 15 minutes. The results are presented in Table~\ref{tab:timebudget}.

\begin{table}[!h]
\centering
\caption{GED estimation error (MAE) on \code under varying time budgets (minutes).}
\label{tab:timebudget}
\begin{tabular}{lccc}
\toprule
Method        & 5 min & 10 min & 15 min \\
\midrule
\btight  & 13.91 & 13.87 & 13.88 \\
\adjip         & 6.98  & 5.05  & 3.96  \\
\compact   & 24.14 & 8.40  & 6.10  \\
\lpged            & 16.31 & 6.28  & 7.72  \\
\ipfp          & 6.44  & 6.47  & 6.39  \\
\bottomrule
\end{tabular}
\end{table}

Branch-Tight and IPFP converged within 5 minutes, as evidenced by the absence of any improvement in MAE with larger time budgets. The remaining three methods exhibited modest gains when given additional time, suggesting that they benefit from prolonged optimization. Still, Eugene achieves a MAE of $0.75$ within 7 minutes, outperforming all baselines even at the maximum allotted time.}

\definecolor{2nd}{rgb}{0.7,1,0.5}
\definecolor{1st}{rgb}{0.3,1,0}
\definecolor{3rd}{rgb}{0.9,1,0.7}

% \definecolor{3rd}{cmyk}{0, 0, 0, 0.2}
% \definecolor{2nd}{cmyk}{0, 0, 0, 0.4}
% \definecolor{1st}{cmyk}{0, 0, 0, 0.6}

 \begin{table}[!h]
 %\vspace{-0.1in}
\caption{\scalebox{1}{Accuracy (MAE) of \name vs. \name'.}}\label{tab:ablation}
\centering
\scalebox{0.95}{
    \begin{tabular}{p{1.3cm}|p{1.3cm}p{1.3cm}|p{1.3cm}p{1.3cm}} 
      \toprule
      & \multicolumn{2}{c|}{Cost Setting Case 1} & \multicolumn{2}{c}{Cost Setting Case 2}\\
      \textbf{Datasets} & \name & \name' & \name & \name'\\
      \midrule
      \aids & \textbf{0.33} & 10.51 & \textbf{0.58} & 9.22\\
      \molhiv & \textbf{0.65} & 9.96 & \textbf{0.79} & 11.63\\
      \code & \textbf{0.75} & 13.46 & \textbf{0.58} & 6.04\\
      \mutag & \textbf{0.68} & 19.12 & \textbf{1.01} &16.10 \\
      \bottomrule
    \end{tabular}}
    \vspace{-0.1in}
 \end{table}

  \begin{table}[!h]
 %\vspace{-0.1in}
\caption{\scalebox{1}{Accuracy (MAE) of \name vs. \name-NoIR.}}\label{tab:ablation2}
\centering
\scalebox{0.95}{
    \begin{tabular}{p{1.3cm}|p{1.3cm}p{2.3cm}|p{1.3cm}p{2.3cm}} 
      \toprule
      & \multicolumn{2}{c|}{Cost Setting Case 1} & \multicolumn{2}{c}{Cost Setting Case 2}\\
      \textbf{Datasets} & \name & \name-NoIR & \name & \name-NoIR\\
      \midrule
      \aids & \textbf{0.33} & 0.80 & \textbf{0.58} & 1.15\\
      \molhiv & \textbf{0.65} & 1.16 & \textbf{0.79} & 1.57\\
      \code & \textbf{0.75} & 1.19 & \textbf{0.58} & 1.02\\
      \mutag & \textbf{0.68} & 1.14 & \textbf{1.01} & 1.53\\
      \bottomrule
    \end{tabular}}
    \vspace{-0.1in}
 \end{table}
\definecolor{1st}{rgb}{0.8, 1, 0.5}
\begin{table}[!h] % we need !h for compression purposes
\caption{Accuracy comparison of \name with \name-\textsc{FW} in Cost Setting 1}\label{tab:fw}
\centering
\scalebox{0.9}{
    \begin{tabular}{p{2.6cm}|p{1.6cm}p{1.6cm}p{1.6cm}p{1.6cm}} 
      \toprule
      \textbf{Methods} & \aids & \molhiv & \code & \mutag\\
      \name-\textsc{FW} & 6.67 & 11.79 & 6.59 & 13.09\\
      \name & \cellcolor{1st}{0.33} & \cellcolor{1st}{0.65} & \cellcolor{1st}{0.75} & \cellcolor{1st}{0.68} \\
      \bottomrule
    \end{tabular}}
\end{table}
\subsubsection{Ablation Study}\label{app:ablation}
We have so far evaluated \name, which refines a doubly stochastic matrix toward a quasi-permutation matrix using a permutation-inducing regularizer before rounding. For comparison, we introduce a variant, \name', which \emph{directly rounds} the doubly stochastic solution without this regularization. As shown in Table~\ref{tab:ablation}, \name yields substantially lower MAE, highlighting the benefit of guiding the solution closer to a permutation before rounding.

\iclr{We also assess the impact of the inverse relabelling strategy of \madam, which recenters the problem after each iteration. To this end, we define a variant, \name-NoIR, that omits this transformation. Table~\ref{tab:ablation2} reports MAE for both variants: \name consistently outperforms \name-NoIR, demonstrating the importance of performing gradient updates in coordinates aligned with the identity.}

\nips{We also investigate the effect of using the Frank-Wolfe (FW) algorithm in place of Adam within Algorithm~\ref{M-Adam}. As shown in Table~\ref{tab:fw}, the M-Adam variant significantly outperforms the version that employs FW (\name-FW), demonstrating the effectiveness of our optimizer choice.}

\definecolor{1st}{rgb}{0.8, 1, 0.5}
\begin{table*}[t] % we need !h for compression purposes
\caption{Accuracy comparision with varying $\mu$ }\label{tab:param_mu}
 \vspace{-2mm}
\centering
\scalebox{0.9}{
    \begin{tabular}{p{1cm}|p{1.6cm}p{1.6cm}p{1.6cm}p{1.6cm}} 
      \toprule
      \textbf{$\mu$} & \aids & \molhiv & \code & \mutag\\
      \midrule
      0.1 & 3.07 & 9.13 & 4.97 & 5.31\\
      0.2 & 2.29 & 7.5 & 2.96 & 3.98\\
      0.5 & 0.9 & 3.39 & 0.9 & 1.36\\
      1 & \cellcolor{1st}{0.58} & \cellcolor{1st}{0.79} & \cellcolor{1st}{0.58} & \cellcolor{1st}{1.01}\\
      2 & 0.85 & 1.14 & 0.84 & 1.69\\
      \bottomrule
    \end{tabular}}
    % \vspace{-0.1in}
\end{table*}

\begin{table*}[t] % we need !h for compression purposes
\caption{Accuracy comparision with varying $\alpha$}\label{tab:param_alpha}
 \vspace{-2mm}
\centering
\scalebox{0.9}{
    \begin{tabular}{p{1cm}|p{1.6cm}p{1.6cm}p{1.6cm}p{1.6cm}} 
      \toprule
      \textbf{$\alpha$} & \aids & \molhiv & \code & \mutag\\
      \midrule
      0.1 & 0.61 & 0.82 & 0.62 & \cellcolor{1st}{0.98}\\
      0.01 & 0.58 & 0.81 & 0.61 & 1.02\\
      0.001 & \cellcolor{1st}{0.58} & \cellcolor{1st}{0.79} & \cellcolor{1st}{0.58} & 1.01\\
      \bottomrule
    \end{tabular}}
    % \vspace{-0.1in}
\end{table*}
\subsubsection{Parameter Sensitivity}\label{app:param_sensitivity}
\nips{We analyze the sensitivity of the M-Adam algorithm to the parameters listed in Appendix~\ref{app:parameters}, as shown in Tables~\ref{tab:param_mu} and~\ref{tab:param_alpha}. A lower value of $\mu$ increases the weight of edge costs, whereas a higher $\mu$ prioritizes node costs. Across all datasets, $\mu = 1$ yields the best performance. We use $\alpha = 0.001$ (the default value for Adam), which performs best on three out of four datasets.}

\iclr{To examine the impact of the $\lambda$-scheduling in \madam, we conducted experiment where the increment step was varied, results are presented in Table~\ref{tab:lambda_schedule}

\begin{table}[!h]
\centering
\caption{Effect of varying $\lambda$-increment step on GED estimation error (MAE).}
\label{tab:lambda_schedule}
\begin{tabular}{lcccc}
\toprule
Increment step & AIDS & molhiv & code2 & Mutag \\
\midrule
0.1 & 1.45 & 2.08 & 1.40 & 2.10 \\
0.5 & \textbf{0.33} & \textbf{0.65} & \textbf{0.75} & \textbf{0.68} \\
1   & 0.80 & 1.54 & 2.77 & 1.85 \\
2   & 3.19 & 6.18 & 10.01 & 9.88 \\
\bottomrule
\end{tabular}
\end{table}

\begin{itemize}
    \item \textbf{Increment = 0.1:} The influence of permutation constraints remained weak throughout optimization, leading to under-constrained solutions and suboptimal performance.
    \item \textbf{Increment = 0.5:} This yielded the best results, striking a balance between exploration and constraint enforcement, and was adopted as the default setting in Eugene.
    \item \textbf{Increment = 1, 2:} The optimizer rapidly enforced hard permutation constraints, prematurely narrowing the search space and degrading solution quality.
\end{itemize}

These results emphasize the importance of a carefully tuned $\lambda$-schedule in achieving both accuracy and stability in GED estimation.}

\begin{figure*}[!h]
% \vspace{-2mm}
\centering
\subfloat[\name]{\includegraphics[width =1.4in]{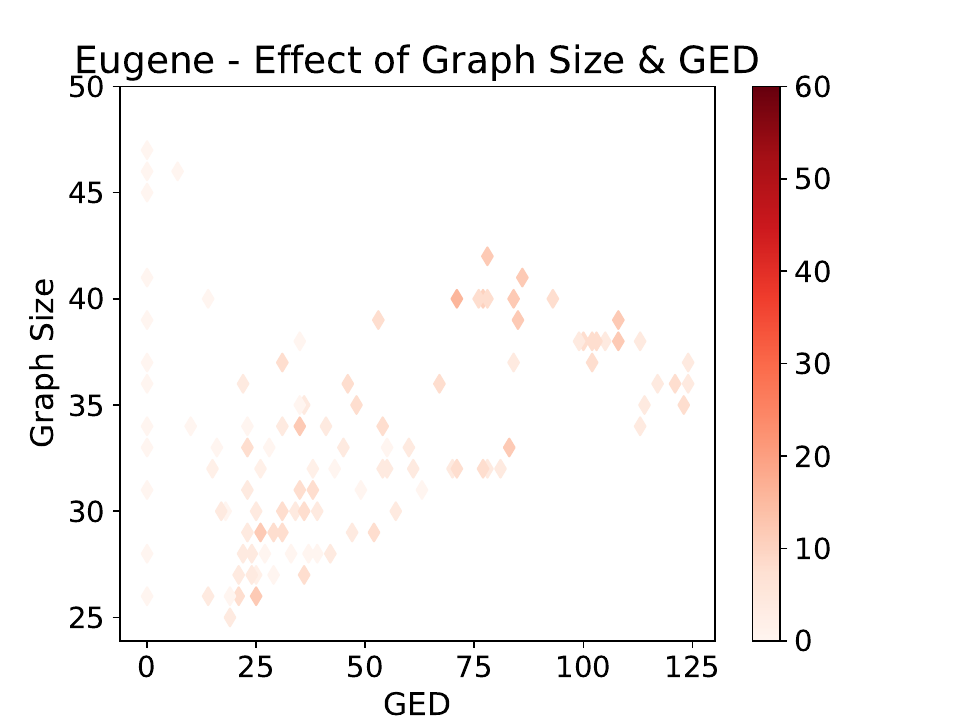}}
\subfloat[\egsc]{\includegraphics[width=1.4in]{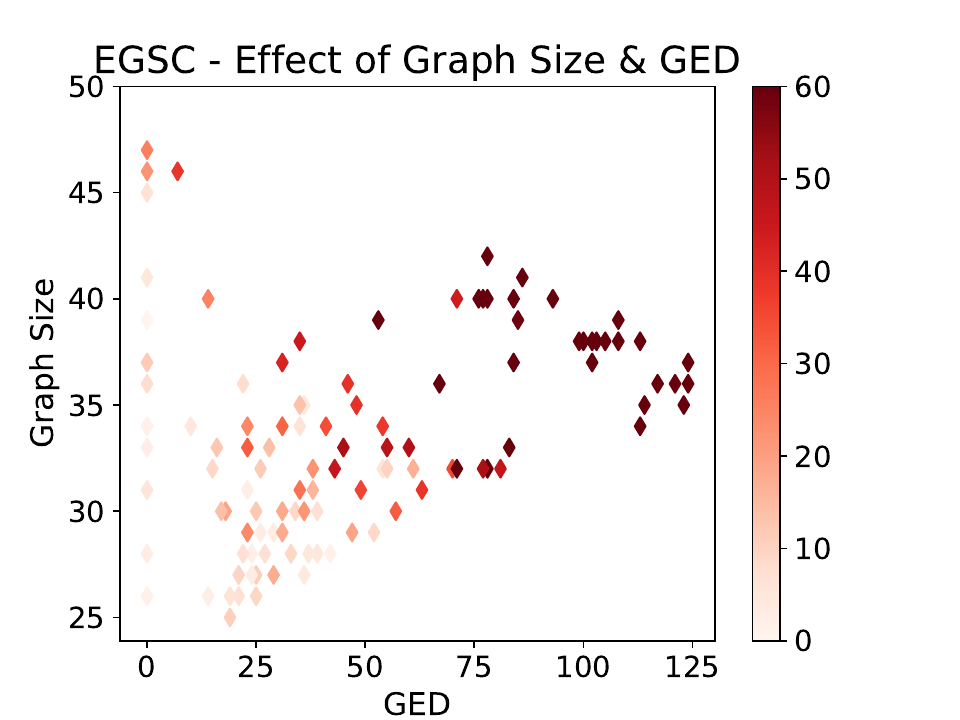}}
\subfloat[\hmn]{\includegraphics[width =1.4in]{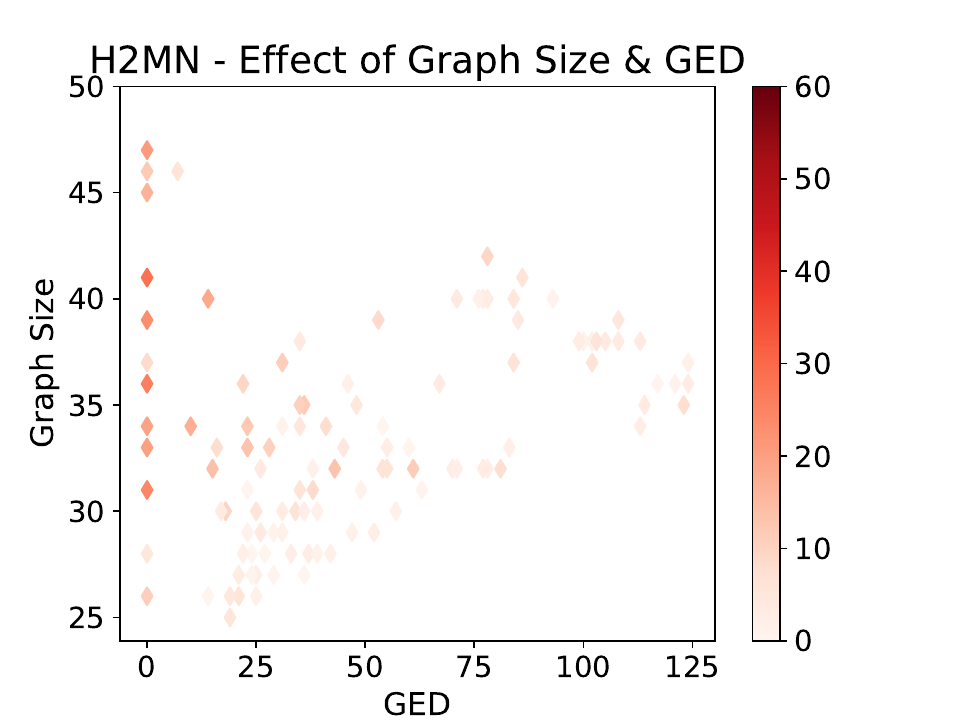}}
\subfloat[\graphedx]{\includegraphics[width =1.4in]{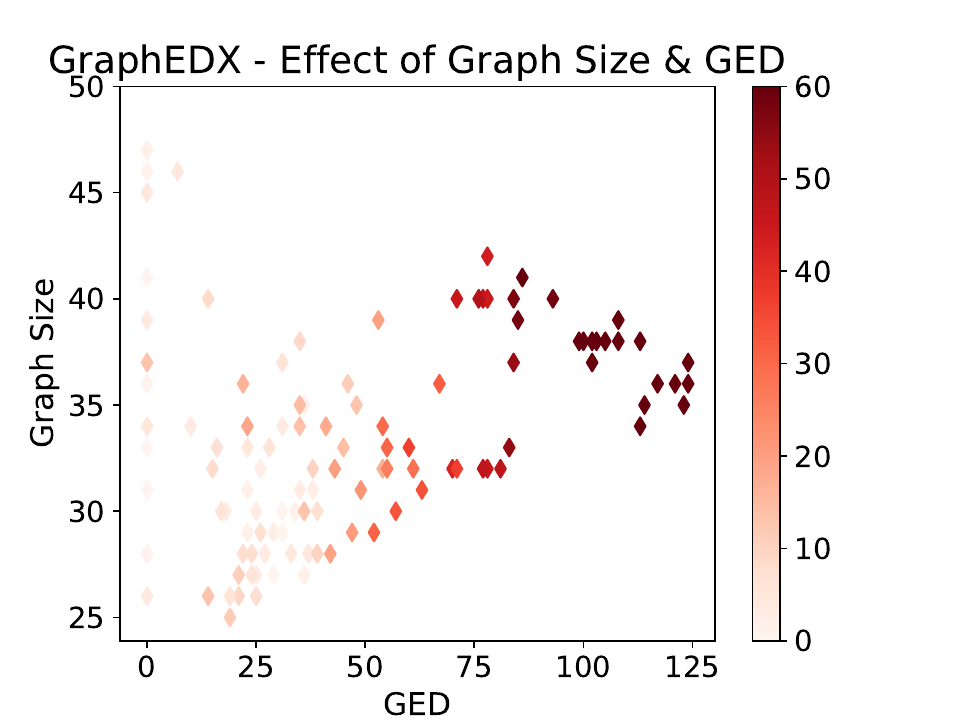}}
% \vspace{-3mm}
\caption{MAE heatmap vs. graph size \& GED for \aids for graphs of size $[25, 50]$.}\label{fig:heatmap_aids}
\vspace{-4mm}
\end{figure*}
\begin{figure*}[!h]
% \vspace{-2mm}
\centering
\subfloat[\name]{\includegraphics[width =1.4in]{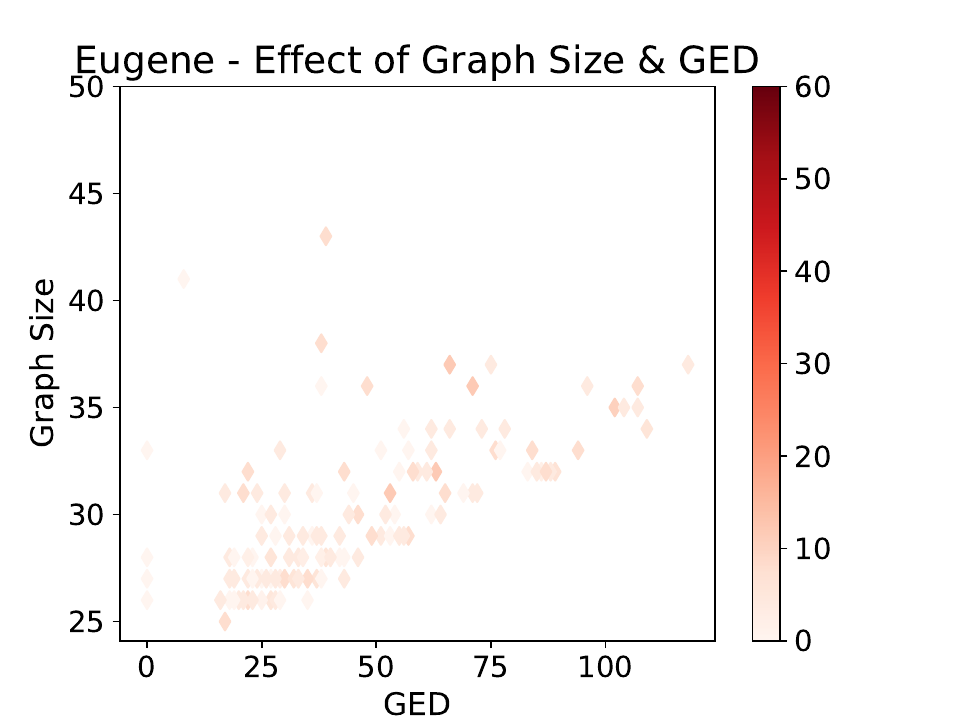}}
\subfloat[\egsc]{\includegraphics[width=1.4in]{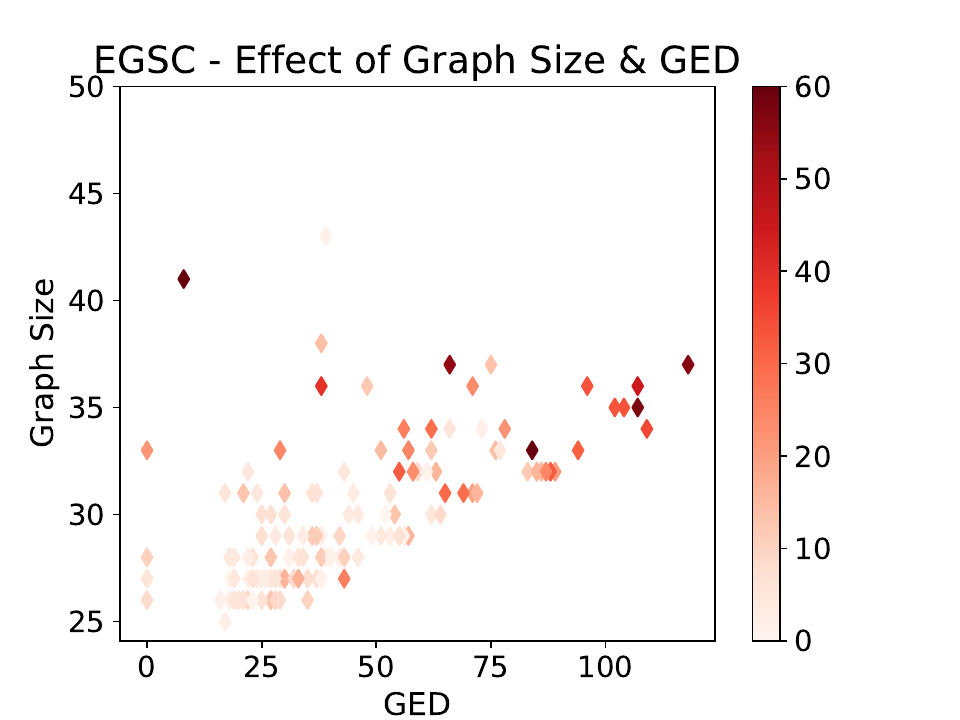}}
\subfloat[\hmn]{\includegraphics[width =1.4in]{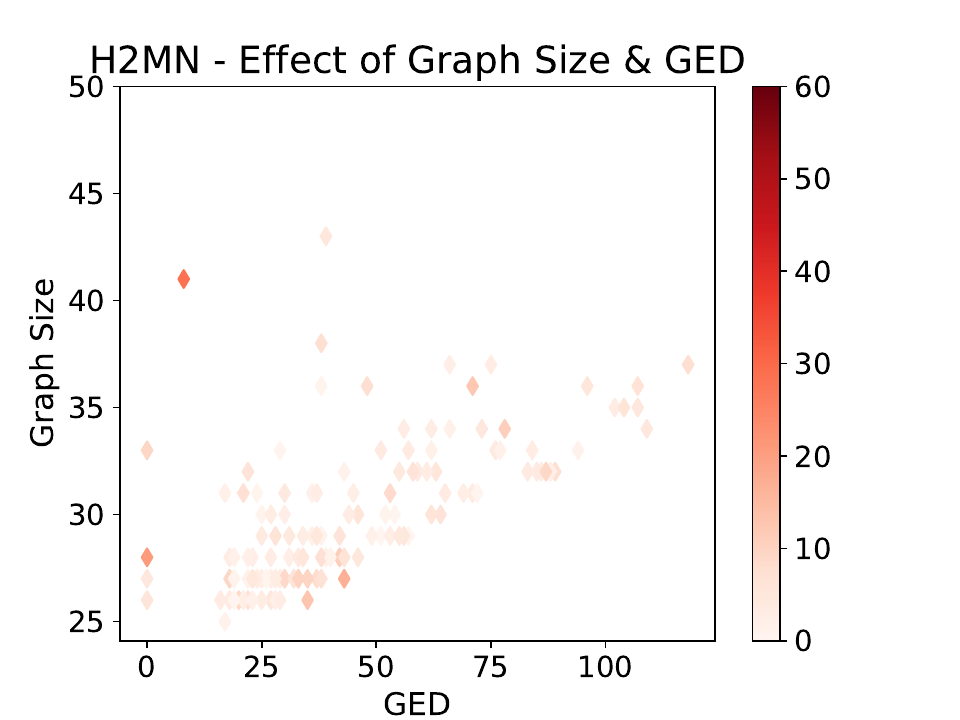}}
\subfloat[\graphedx]{\includegraphics[width =1.4in]{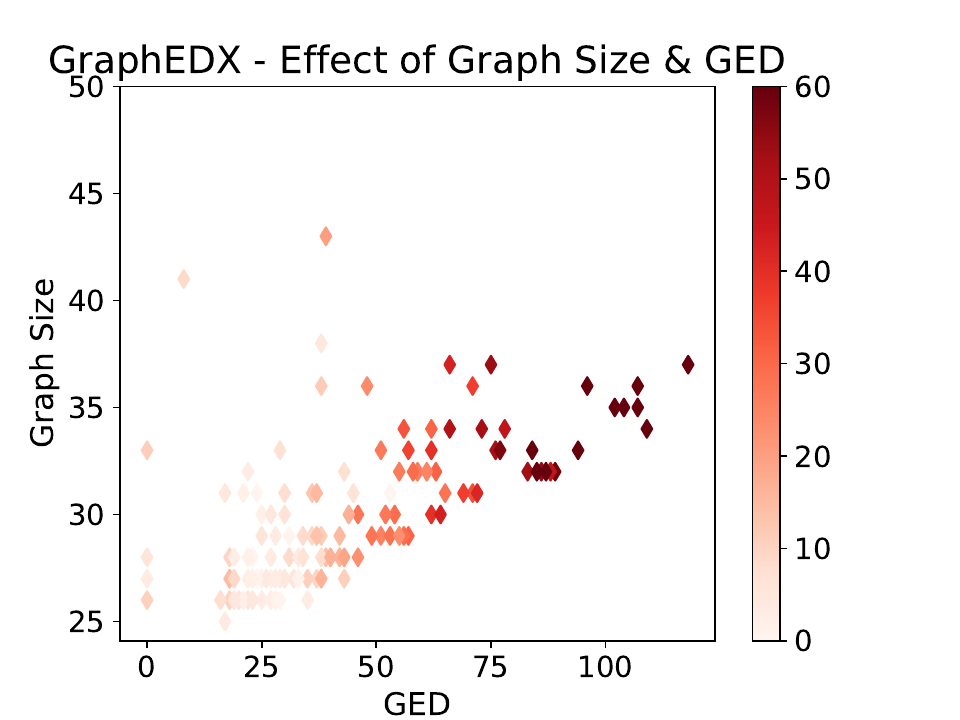}}
\vspace{-3mm}
\caption{MAE heatmap vs. graph size \& GED for \molhiv for graphs of size $[25, 50]$.}\label{fig:heatmap_molhiv}
\vspace{-4mm}
\end{figure*}
\begin{figure*}[!h]
% \vspace{-2mm}
\centering
\subfloat[\name]{\includegraphics[width =1.4in]{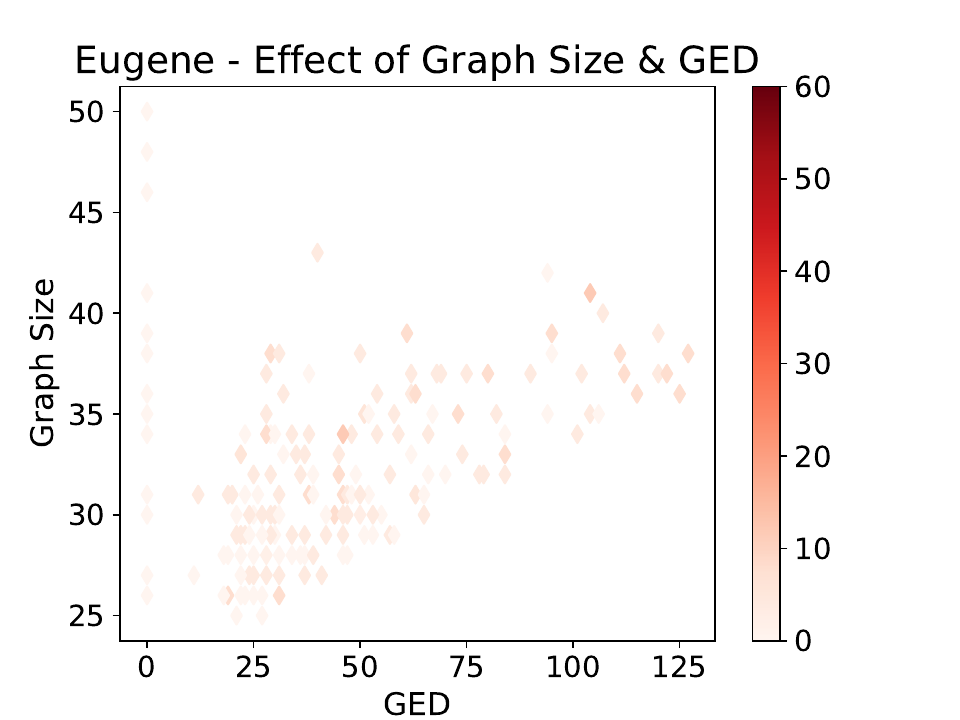}}
\subfloat[\egsc]{\includegraphics[width=1.4in]{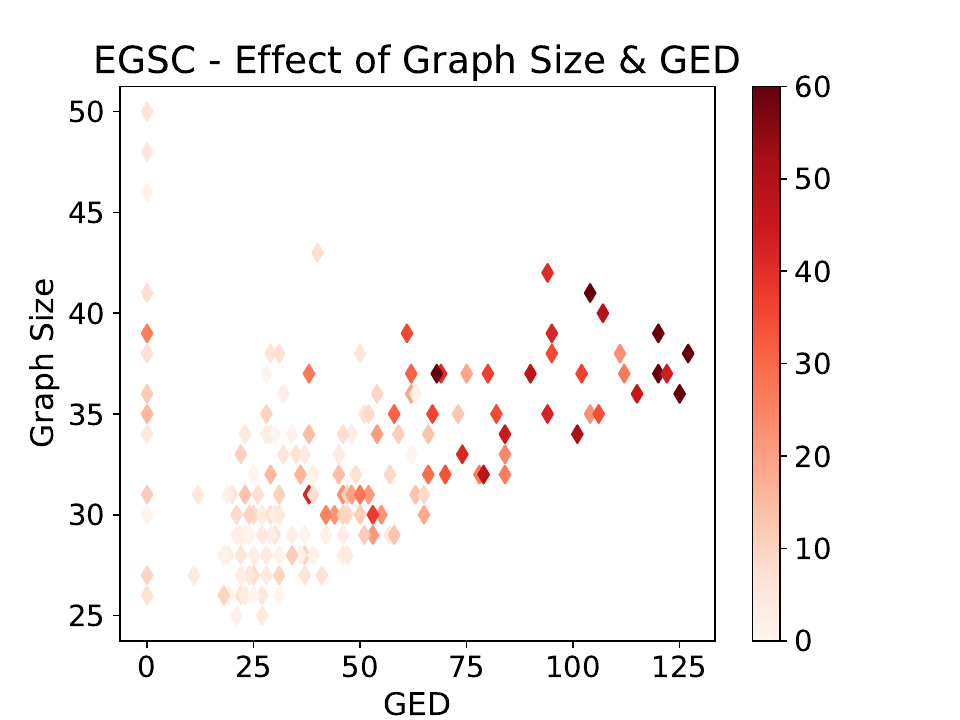}}
\subfloat[\hmn]{\includegraphics[width =1.4in]{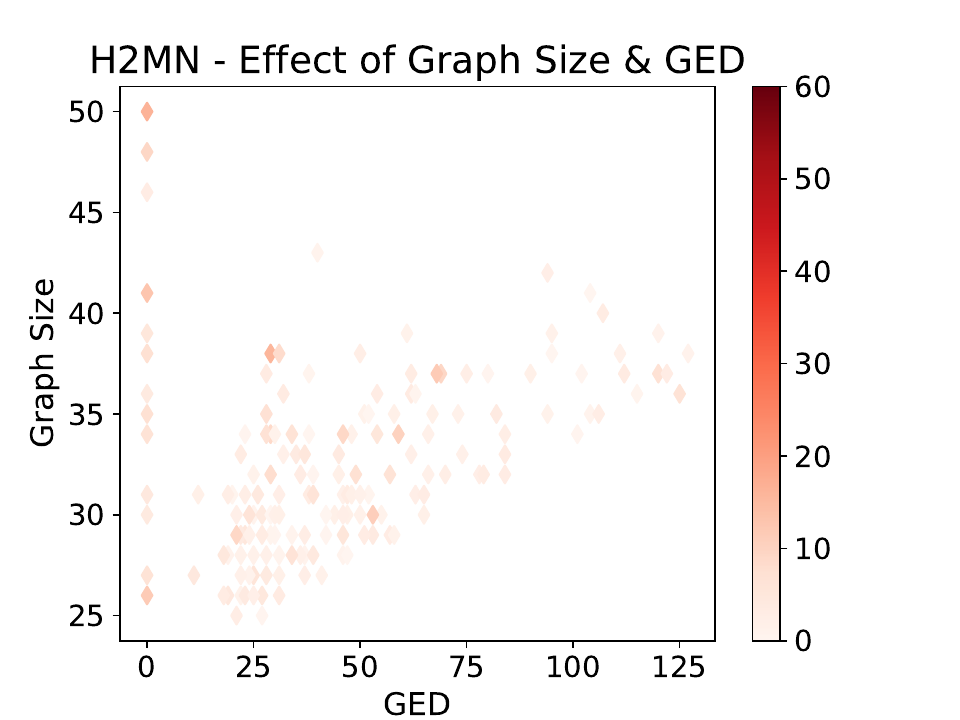}}
\subfloat[\graphedx]{\includegraphics[width =1.4in]{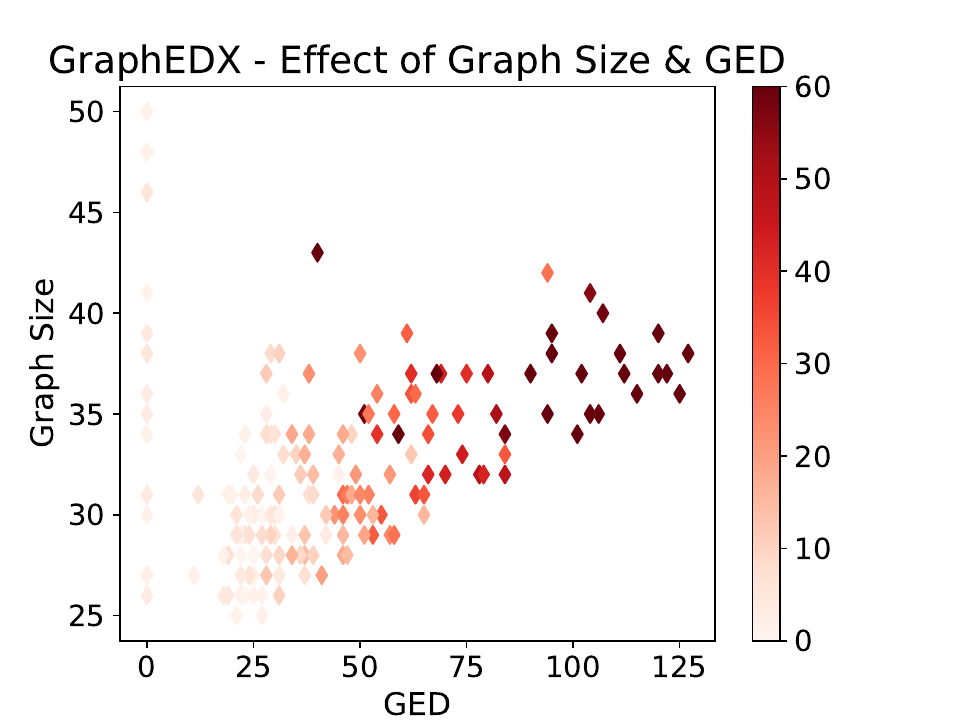}}
% \vspace{-3mm}
\caption{MAE heatmap vs. graph size \& GED for \mutag for graphs of size $[25, 50]$.}\label{fig:heatmap_mutag}
% \vspace{-4mm}
\end{figure*}
\subsubsection{Impact of Graph Size and GED}\label{app:heatmaps}
Section~\ref{sec:large_acc} presented heatmaps of MAE vs. graph size and true GED value on the \code dataset. Heatmaps for the \aids, \molhiv, and \mutag datasets are provided in Figs.~\ref{fig:heatmap_aids}-~\ref{fig:heatmap_mutag}. The conclusions remain consistent: \graphedx, \egsc, and \hmn exhibit noticeably darker tones across the spectrum compared to \name, highlighting \name's superior scalability with respect to GED and graph sizes across datasets.

\nips{\subsubsection{Accuracy on Very Large Graphs}\label{app:vlarge_graphs}
As detailed in ~\cite{btight}, GED methods are traditionally applied to small-scale graphs due to computational complexity. We extend the feasibility of GED approximation to substantially larger graphs. We present results on two unlabelled thousand-scale collaboration network datasets, Netscience ($|V|=379$, $|E|=914$) and HighSchool ($|V|=327$, $|E|=5818$) in Table~\ref{tab:vlarge_graphs}. To our knowledge, no prior GED approximation benchmark handles graphs of this scale. On HighSchool, a dense evolving dataset, we compute the GED of the last graph version from versions containing 80\%, 85\%, 90\%, and 99\% of edges. On NetScience, we create five graphs by introducing small noise to the original graph. Since it's not feasible to create a training set with exact ground-truth GED for such large graphs, we excluded neural models from our analysis. \ipfp didn't terminate within a time limit of 3 hrs. Results clearly indicate superior scalabilty of \name both in terms of MAE and running times.}
\definecolor{1st}{rgb}{0.8, 1, 0.5}
\begin{table}[H] % we need !h for compression purposes
\caption{Performance comparison on HighSchool and NetScience Datasets}\label{tab:vlarge_graphs}
 % \vspace{-2mm}
\centering
\scalebox{0.9}{
    \begin{tabular}{p{3.4cm}|p{2cm}p{2cm}|p{2cm}p{2cm}} 
      \toprule
      & \multicolumn{2}{c|}{MAE} & \multicolumn{2}{c}{Running Time (sec)}\\
      \textbf{Methods} & HighSchool & NetScience & HighSchool & NetScience\\
      \midrule
      \adjip & 4568 & 152.99 & 2695 & 1446 \\
      \btight & 582 & 833 & 1115 & 2369 \\
      \lpged & 5032 & 859.4 & 1912 & 1526 \\
      \midrule
      \name & \cellcolor{1st}{0} & \cellcolor{1st}{22.8} & \cellcolor{1st}{961} & \cellcolor{1st}{1372} \\
      \bottomrule
    \end{tabular}}
    % \vspace{-0.1in}
\end{table}

\iclr{\subsubsection{Illustrative Example with domain-specific edit costs}\label{app:domaincosts}
To model meaningful structural similarity, we design edit costs with domain-specific heuristics from chemistry for \molhiv dataset.

\paragraph{Node Substitution Cost.} Substituting one atom for another alters a molecule’s electronic properties, reactivity, and biological function. To account for these effects, node substitution costs are assigned based on the electronegativity difference between atoms:
\begin{itemize}
    \item \textbf{Low Cost (1):} Applied when the electronegativity difference is less than $0.2$. These substitutions typically involve chemically similar atoms that frequently co-occur in analogous functional groups.
    \item \textbf{Moderate Cost (2):} Assigned when the difference lies in $[0.2, 0.7]$, indicating moderate chemical dissimilarity.
    \item \textbf{High Cost (3):} Used when the difference exceeds $0.7$, reflecting substitutions likely to disrupt molecular structure and activity.
\end{itemize}

\paragraph{Node Insertion / Deletion Cost.} The cost of inserting or deleting a node is determined by the bond multiplicity of the associated atom:
\begin{itemize}
    \item \textbf{Cost = 3:} Atom participates in at least one triple bond.
    \item \textbf{Cost = 2:} Atom participates in at least one double bond but no triple bond.
    \item \textbf{Cost = 1:} Atom is involved only in single bonds.
\end{itemize}
This hierarchy reflects the increasing structural and energetic disruption when removing atoms from more rigid bonding environments.

\paragraph{Edge Insertion / Deletion Cost.} Edge insertion and deletion costs are set uniformly to~\textbf{1}.

\begin{table}[!h]
\centering
\caption{GED estimation error (MAE) on \molhiv under chemistry-informed edit costs.}
\label{tab:chem_editcost}
\begin{tabular}{lccccc}
\toprule
Method & \name & \eric & \egsc & \graphedx & \greed \\
\midrule
\molhiv & \textbf{1.30} & 2.09 & 1.94 & 1.74 & 2.58 \\
\bottomrule
\end{tabular}
\end{table}

We presented the results in Table~\ref{tab:chem_editcost}. Eugene outperforms competing baselines under the proposed chemistry-informed edit cost setting, demonstrating its ability to effectively capture real-world molecular similarity.
}

\nips{\subsubsection{Illustrative Example of \name’s Pipeline}
\label{app:domainedit}
We considered two graphs of sizes 12, 11 respectively and show four stages of \name’s operation: \begin{enumerate*}[label=(\roman*)] \item the initial mapping; \item the doubly stochastic matrix generated after the first iteration of Algorithm~\ref{M-Adam} ($\lambda$ = 0); \item the quasi-permutation matrix at the end of third iteration of Algorithm~\ref{M-Adam}; \item The final mapping returned by \name \end{enumerate*}. The optimal transformation from Graph 1 to Graph 2 involves removing node 10, removing the edge from node 1 to node 4, and adding an edge from node 5 to node 9 in Graph 1. As the figure shows, by the third iteration, our novel regularizer has turned the doubly stochastic matrix to a sparse one. At the end, the algorithm achieves the \emph{optimal} node alignment. After iteration 3, nodes 6 and 7 of Graph 1 have similar weightage for nodes 5 and 6 of Graph 2, as these nodes share similar structural neighborhoods. Node 10 is mapped to node 11, which is a dummy node in Graph 2, indicating that it should be deleted.}
\begin{figure}[H]
\centering
% \vspace{0.10in} % do not remove space on top of column
\includegraphics[width=6in]{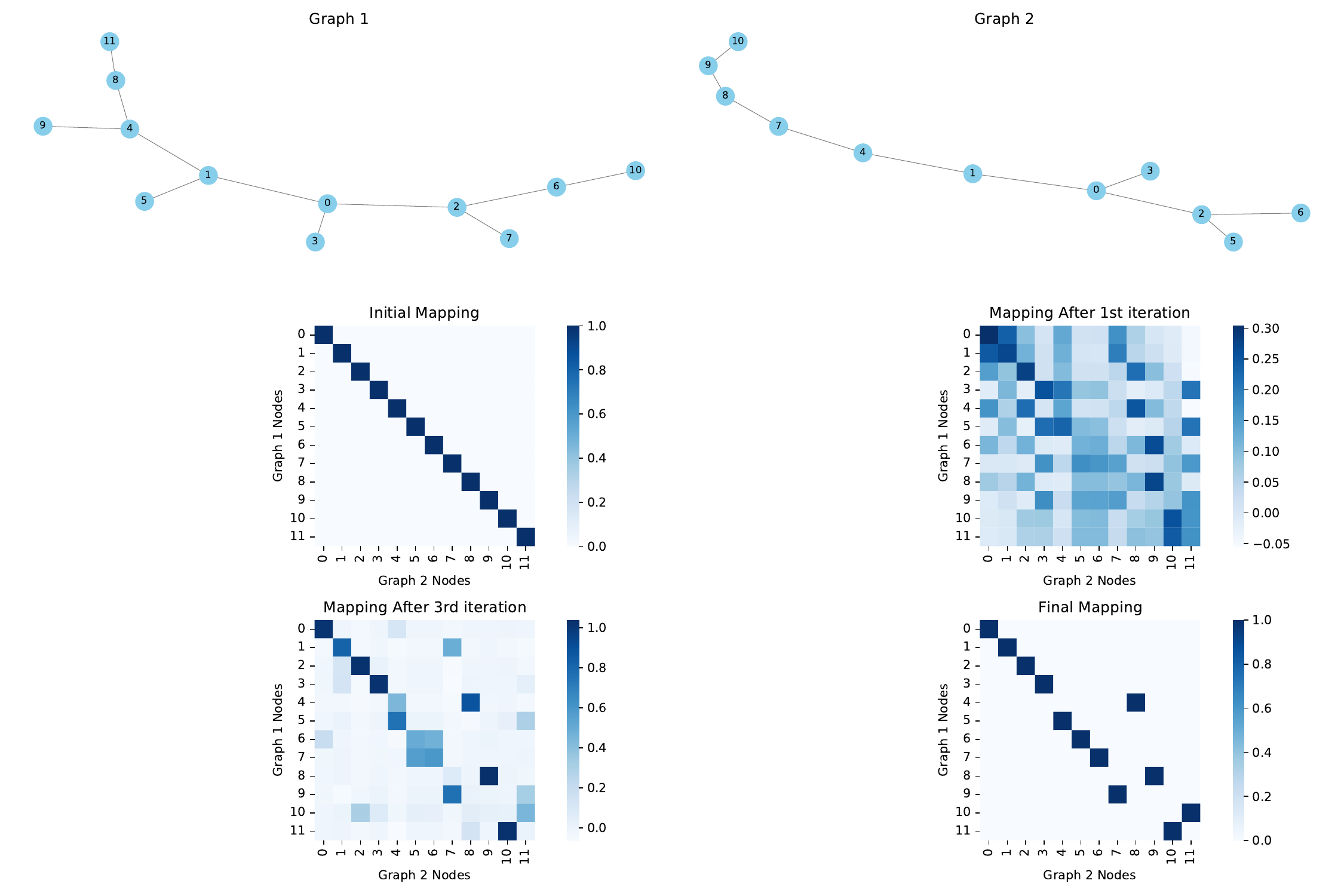}
\caption{Operational stages of \name}\label{fig:toy_example}
\end{figure}

%\input{sections/checklist}
% \clearpage
% \input{sections/icml_rebuttal}
%\clearpage
%\section{Appendix}
%\input{sections/00.appendix}
% \input{resubmission_report}
\end{document}